\documentclass[11pt]{article}

\usepackage{comment,url,algorithm,algorithmic,graphicx,subcaption,relsize}
\usepackage{amssymb,amsfonts,amsmath,amsthm,amscd,dsfont,mathrsfs,mathtools,nicefrac}
\usepackage{float,psfrag,epsfig,color,xcolor,url,hyperref}
\usepackage{epstopdf,bbm,mathtools,enumitem}
\usepackage[toc,page]{appendix}
\usepackage[mathscr]{euscript}
\usepackage{xspace}

\usepackage[top=1in, bottom=1in, left=1in, right=1in]{geometry}

\def\balign#1\ealign{\begin{align}#1\end{align}}
\def\baligns#1\ealigns{\begin{align*}#1\end{align*}}
\def\balignat#1\ealign{\begin{alignat}#1\end{alignat}}
\def\balignats#1\ealigns{\begin{alignat*}#1\end{alignat*}}
\def\bitemize#1\eitemize{\begin{itemize}#1\end{itemize}}
\def\benumerate#1\eenumerate{\begin{enumerate}#1\end{enumerate}}

\newenvironment{talign*}
 {\csname align*\endcsname}
 {\endalign}
\newenvironment{talign}
 {\csname align\endcsname}
 {\endalign}

\def\balignst#1\ealignst{\begin{talign*}#1\end{talign*}}
\def\balignt#1\ealignt{\begin{talign}#1\end{talign}}

\let\originalleft\left
\let\originalright\right
\renewcommand{\left}{\mathopen{}\mathclose\bgroup\originalleft}
\renewcommand{\right}{\aftergroup\egroup\originalright}

\def\tinycitep*#1{{\tiny\citep*{#1}}}
\def\tinycitealt*#1{{\tiny\citealt*{#1}}}
\def\tinycite*#1{{\tiny\cite*{#1}}}
\def\smallcitep*#1{{\scriptsize\citep*{#1}}}
\def\smallcitealt*#1{{\scriptsize\citealt*{#1}}}
\def\smallcite*#1{{\scriptsize\cite*{#1}}}

\def\mbb#1{\mathbb{#1}}

\def\R{\mathbb{R}}

\def\<{\left\langle} %
\def\>{\right\rangle}

\def\E{\mbb{E}} %

\DeclareSymbolFont{rsfs}{U}{rsfs}{m}{n}
\DeclareSymbolFontAlphabet{\mathscrsfs}{rsfs}

\providecommand{\argmin}{\mathop\mathrm{arg min}}

\ifdefined\nonewproofenvironments\else
\ifdefined\ispres\else
\newtheorem{theorem}{Theorem}
\newtheorem{lemma}[theorem]{Lemma}
\newtheorem{corollary}[theorem]{Corollary}

\renewenvironment{proof}{\noindent\textbf{Proof.}\hspace*{.3em}}{\qed\\}
\newenvironment{proof-sketch}{\noindent\textbf{Proof Sketch}
  \hspace*{1em}}{\qed\bigskip\\}
\newenvironment{proof-idea}{\noindent\textbf{Proof Idea}
  \hspace*{1em}}{\qed\bigskip\\}
\newenvironment{proof-of-lemma}[1][{}]{\noindent\textbf{Proof of Lemma {#1}}
  \hspace*{1em}}{\qed\\}
\newenvironment{proof-of-theorem}[1][{}]{\noindent\textbf{Proof of Theorem {#1}}
  \hspace*{1em}}{\qed\\}
\newenvironment{proof-attempt}{\noindent\textbf{Proof Attempt}
  \hspace*{1em}}{\qed\bigskip\\}

\fi

\newtheorem{proposition}[theorem]{Proposition}

\newtheorem{assumption}{Assumption}
\fi
\makeatletter
\@addtoreset{equation}{section}
\makeatother

\hypersetup{
  colorlinks,
  linkcolor={red!50!black},
  citecolor={blue!50!black},
  urlcolor={blue!80!black}
}

\mathtoolsset{showonlyrefs}

\allowdisplaybreaks[4]

\newcommand{\rmd}{\mathrm d}

\definecolor{darkmidnightblue}{rgb}{0.0, 0.2, 0.4}
\definecolor{darkpowderblue}{rgb}{0.0, 0.2, 0.6}
\definecolor{dukeblue}{rgb}{0.0, 0.0, 0.61}

\hypersetup{
    colorlinks = true,
    citecolor= midnightblue,
    urlcolor= black,
    breaklinks=true,
    linkcolor = midnightblue,
    linkbordercolor = {white},
}

\definecolor{darkmidnightblue}{HTML}{003366}    
\definecolor{midnightblue}{HTML}{0059b3}
\definecolor{chromered}{HTML}{f14233}

\begin{document}

\title{On the Limits of Latent Reuse in Diffusion Models
}

 \author{
 Yifeng Yu\thanks{Department of Mathematical Sciences, Tsinghua University \texttt{yyf22@mails.tsinghua.edu.cn}}
 \and
 Lu Yu\thanks{
  Department of Data Science,
  City University of Hong Kong \texttt{lu.yu@cityu.edu.hk}
 }
}

\maketitle

\begin{abstract}
Diffusion models are often trained in low-dimensional latent spaces, which are then reused for related but shifted datasets. In this work, we study when such latent reuse remains reliable under distribution shift. We consider a source-target setting in which both datasets are approximately low-dimensional but may lie near different subspaces. We show that freezing and reusing a source latent space induces a target-domain score error governed by two quantities: the principal-angle misalignment between the source and target subspaces, and the target ambient noise amplified by the diffusion time scale. Motivated by these limits, we further study mixed source-target training and characterize how the required shared latent dimension depends on the relative geometry of the two distributions. 
Our results provide theoretical guidance on when latent reuse is reliable and when learning a shared representation may be necessary.
\end{abstract}

 \tableofcontents

\section{Introduction}

{Diffusion models} have recently achieved remarkable breakthroughs across multiple domains~\cite{ramesh2022hierarchical,popov2021grad,xu2022geodiff}. 
Their success is driven by a denoising principle: a forward process gradually corrupts data with Gaussian noise, while a learned score function is used to reverse this process and generate new samples~\cite{song2020score}.
In many modern applications, however, diffusion models are not trained directly in the original ambient space. 
Instead, high-dimensional data such as high-resolution images are first compressed by an encoder, often the encoder of a variational autoencoder (VAE), and the diffusion model is trained in the resulting lower-dimensional latent space~\cite{rombach2022high,podell2023sdxl,chen2023pixart,chen2024pixart,wu2025improved}. 
This strategy is computationally attractive and is consistent with the empirical observation that many high-dimensional datasets exhibit strong low-dimensional structure~\cite{pope2021intrinsic}. 
After pretraining, the representation map is often kept fixed and reused for related datasets, while a new score model is trained in the induced latent coordinates~\cite{chen2023pixart,chen2024pixart,leng2025repa,wu2025improved}.

While convenient, latent reuse raises a fundamental question:
\emph{when does a latent space learned from a source distribution remain reliable under target shift?}
The answer is not obvious, since a representation that is effective for the source distribution need not preserve all directions relevant to the target score.

Existing theory provides important insight into why diffusion models can exploit low-dimensional structure. 
A growing line of work establishes score approximation, score estimation, and distribution recovery guarantees when the data lie on, or near, an unknown low-dimensional subspace or manifold, with rates depending primarily on the intrinsic dimension rather than the ambient dimension~\cite{Chen2023_scorelowdim,li2024adapting,wang2025diffusion,hu2024statistical,oko2023diffusion,azangulov2024convergence,tang2024adaptivity,chen2025diffusion,yakovlev2025generalization,liang2025low,huang2026denoising,achilli2024losing}. 
These results, however, are largely single-distribution in nature: they explain the benefits of low-dimensional modeling, while leaving open how a learned representation behaves when reused under distribution shift.

A separate line of work studies transfer and adaptation of diffusion models across domains~\cite{cheng2025provable,ouyang2024transfer,song2025domain,kleutgens2025guided,moon2022fine,xie2023difffit,han2023svdiff}. 
For example, recent work represents the target-domain score through a pretrained source-domain score together with additional guidance~\cite{ouyang2024transfer,kleutgens2025guided}, while other work establishes sample-efficient transfer guarantees for conditional diffusion models through the lens of representation learning~\cite{cheng2025provable}. 
These works show how pretrained diffusion models or learned representations can be adapted across related tasks. 
Complementary to this perspective, we study a different widely used practice: freezing a data latent space learned from a source distribution and reusing it as the coordinate system for score estimation on a shifted target distribution.
This distinction is important because latent reuse can fail for reasons that are invisible in existing single-distribution low-dimensional theory. 
Even when both source and target distributions are approximately low-dimensional, the source latent space may be poorly aligned with the target signal subspace, and the target distribution may contain ambient noise outside the reused representation. 
In this case, no downstream score model trained only on the reused latent coordinates can fully recover the missing target directions.

Motivated by this gap, we take a first step toward a mathematical understanding of latent reuse under source-target distribution shift. 
We consider a model where both distributions are approximately low-dimensional but may concentrate near different subspaces, allowing us to isolate how representation mismatch affects score estimation. 
Our analysis shows that frozen reuse can fail for geometric reasons that cannot be resolved solely by more target samples or a more expressive downstream score model. 
This motivates mixed source--target training as an alternative strategy for learning a shared representation adapted to both distributions. 
Towards this, our contributions can be summarized as follows.
\begin{itemize}
    \item \textbf{Identifying the limits of frozen latent reuse.}
    In Theorem~\ref{prop:frozen_projection_lower_bound}, we show that frozen reuse can incur irreducible target-domain score error, even with an arbitrarily expressive downstream score model. The bound identifies two concrete failure mechanisms: principal-angle misalignment with the target subspace and target ambient noise amplified by the diffusion time scale.

    \item \textbf{Quantifying when frozen reuse is reliable.}
    In Proposition~\ref{prop:frozen_projection_upper_bound} and Theorem~\ref{thm:frozen_projector_statistical_bound}, we construct an explicit frozen-coordinate comparator and derive a finite-sample ReLU score-estimation guarantee. These results separate statistical, approximation, and geometric errors, thereby clarifying when the cost of latent reuse is small and when it remains intrinsic.

    \item \textbf{Mixed training beyond frozen reuse.}
    In Proposition~\ref{prop:optimal_mixed_projector}, Theorem~\ref{thm:mixed_oracle_projector_comparison}, and Corollary~\ref{cor:mixed_containment_regime}, we analyze mixed source--target training through a shared-projector oracle. The results show how learning a shared representation can reduce, and in favorable containment regimes remove, the target signal-mismatch term that persists under source-aligned reuse.
\end{itemize}
Overall, our results suggest that latent reuse should be viewed not only as a downstream score-estimation problem, but also as a representation-level constraint governed by the relative geometry of the source and target low-dimensional structures.
In particular, the bounds show that low-dimensionality alone does not guarantee reliable reuse: alignment, latent dimension, and ambient-noise effects all influence the resulting score error. 
These observations provide a step toward a more precise theoretical understanding of when frozen latent reuse is appropriate and when learning a shared representation may be beneficial.

\noindent\textbf{Notation.}
For a matrix \(A\), let \(\|A\|_{\rm op}\) and \(\|A\|_F\) denote its operator and Frobenius norms. 
For symmetric \(M\), let \(\lambda_{\max}(M)\) and \(\lambda_{\min}(M)\) denote its extreme eigenvalues. 
For orthonormal \(U\), write \(P_U:=UU^\top\) and \(P_U^\perp:=I-P_U\). 
For a metric class \((\mathcal F,d)\), let \(\mathcal N(\epsilon,\mathcal F,d)\) be its \(\epsilon\)-covering number. 
Denote \(a\lesssim b\) if \(a\leqslant Cb\) for an absolute constant \(C\), and use \(\widetilde{\mathcal O}(\cdot)\) to hide logarithmic factors.

\section{Preliminaries}

We begin by introducing the diffusion model used throughout the paper, followed by the source-target setting in which latent representations are learned, frozen, and reused under distribution shift.

\subsection{Diffusion Models}

\noindent\textbf{Framework.} 
We consider the Ornstein--Uhlenbeck forward process
\begin{align}\label{eq:forward_sde}
    \rmd X_t = -\frac{1}{2}X_t\,\rmd t + \rmd W_t,
    \qquad t\in[0,T],
\end{align}
where the initial point $X_0 \sim p_0$ follow an unknown data distribution $p_0$, and $W_t$ denotes a standard $D$-dimensional Brownian motion. 
This process admits the explicit transition kernel
\begin{align}\label{eq:transition_kernel}
    X_t \mid X_0 \sim \mathcal{N}\bigl(\alpha(t)X_0, h(t)I_D\bigr),
\end{align}
with $\alpha(t)=e^{-t/2}, h(t)=1-e^{-t}.$ Thus, the marginal density of $X_t$ is
\begin{align}\label{eq:marginal_density}
    p_t(x)
    =
    \int_{\mathbb{R}^D}
    \mathcal{N}\bigl(x;\alpha(t)x_0,h(t)I_D\bigr)
    p_0(x_0)\,\rmd x_0 .
\end{align}
Sampling from $p_0$ can be performed by reversing the noising process~\cite{anderson1982reverse,cattiaux2023time}. 
The corresponding reverse-time SDE is
\begin{align}\label{eq:reverse_sde}
    \rmd X^{\leftarrow}_t
    =
    \left[
        \frac{1}{2}X^{\leftarrow}_t
        +
        \nabla_x \log p_{T-t}\bigl(X^{\leftarrow}_t\bigr)
    \right]\rmd t
    +
    \rmd \bar W_t,
\end{align}
where $\bar W_t$ is a standard Brownian motion,  independent of $W_t$.
The vector field $\nabla_x\log p_t(x)$ is called the \textit{score function} for $p_t$. 
It is typically unknown, since $p_t$ is not available in closed form.

\noindent \textbf{Score Matching.} In practice, one trains a parametrized score network $s_\theta(x,t)$ to approximate $\nabla_x\log p_t(x)$, typically through an $L^2$ score-matching objective of the form
\begin{align}\label{eq:score_matching_objective}
    \underset{\theta\in\Theta}{\text{minimize}}~~~\E\left[
        \bigl\|
        s_\theta(X_t,t)-\nabla_x \log p_t(X_t)
        \bigr\|_2^2
    \right].
\end{align}
Substituting the learned score estimate into the reverse process~\eqref{eq:reverse_sde} yields the practical continuous-time backward SDE, which is typically approximated by a discretization method (e.g.~\cite{yu2026advancing,chen2023improved,de2022convergence}).

\subsection{Source-target Setting}

In this work, we study latent reuse in a source-target setting. 
The source distribution $p_1$ represents the distribution used to learn the latent representation, while the target distribution $p_2$ represents the shifted distribution on which this representation is reused.
We assume that both distributions are approximately low-dimensional, with each lying near a possibly different linear subspace and perturbed by isotropic ambient noise.
\begin{assumption}\label{asm:noisy_low_dim_shift}
For each domain $i\in\{1,2\}$, a data point $x\sim p_i$ is generated via
\[
   x = A_i z_i + \varepsilon_i,
\]
where $A_i\in\mathbb{R}^{D\times d_i}$ has orthonormal columns, i.e., $A_i^\top A_i=I_{d_i}$, 
$z_i\sim p_{i,z}$ is a latent variable in $\mathbb{R}^{d_i}$, and the ambient noise $\varepsilon_i \sim \mathcal{N}(0,\sigma_i^2 I_D)$
is independent of $z_i$, with $\sigma_i>0$. 

\end{assumption}

Assumption~\ref{asm:noisy_low_dim_shift} extends the exact linear-subspace model commonly used in low-dimensional diffusion theory~\cite{Chen2023_scorelowdim,hu2024statistical,yang2024few} by allowing isotropic ambient noise. 
Here, \(A_i z_i\) represents the intrinsic low-dimensional signal, while \(\varepsilon_i\) captures deviations from \(\operatorname{col}(A_i)\). 
The source and target subspaces are unknown and may differ, allowing us to quantify the effect of latent reuse through subspace misalignment and target ambient noise.

Under this model, Gaussian noising yields an explicit score decomposition.
\begin{lemma}\label{lem:noisy_score_decomp}
    {Let \(X_0\sim p_i\) satisfy Assumption~\ref{asm:noisy_low_dim_shift}.} Under the forward diffusion process $X_t | X_0 \sim \mathcal{N}(\alpha(t)X_0, h(t)I_D)$, define $\tilde{h}_i(t) := \alpha^2(t)\sigma_i^2 + h(t)$, then the exact score function $\nabla\log p_{i,t}(x)$ rigorously decomposes into mutually orthogonal components 
    \begin{align*}
        \nabla\log p_{i,t}(x) = \underbrace{A_i\nabla\log p_{i,t}^{\rm LD}(A_i^\top x)}_{s_\parallel(A_i^\top x,t):\text{ on-support score}} \underbrace{-\dfrac{1}{\tilde{h}_i(t)}(I_D-A_i A_i^\top)x}_{s_\perp(x,t):\text{ ortho. score}},
    \end{align*}
    where 
    $
        p_{i,t}^{\rm LD}(z')=\int_{\R^{d_i}}\phi_{i,t}(z'|z)p_{i,z}(z)\,\rmd  z
 $
    with $\phi_{i,t}(\cdot|z)$ being the Gaussian density function of $\mathcal N(\alpha(t)z, \tilde{h}_i(t)I_{d_i})$.
\end{lemma}

Lemma~\ref{lem:noisy_score_decomp} decomposes the score into a latent component along $\operatorname{col}(A_i)$ and an explicit orthogonal component. 
The latent component depends only on $A_i^\top x$, while the orthogonal component has magnitude governed by the effective variance $\widetilde h_i(t)=\alpha^2(t)\sigma_i^2+h(t)$. 
This decomposition is the basis for our analysis of latent reuse: when a source subspace is frozen and reused on a target distribution, any mismatch between \(\operatorname{col}(A_1)\) and \(\operatorname{col}(A_2)\) directly affects the target score.

For later use, we define $f_i^*(z,t) := \tilde{h}_i(t)\nabla\log p_{i,t}^{\rm LD}(z)+z$. 
Then, the decomposition in Lemma~\ref{lem:noisy_score_decomp} can be written equivalently as
\begin{align}\label{eq:nabla_decomp}
    \nabla\log p_{i,t}(x) = \frac{1}{\tilde{h}_i(t)}A_i f_i^*(A_i^\top x,t) - \frac{1}{\tilde{h}_i(t)}x.
\end{align}

\section{Main Results}

We now present our main theoretical results. 
Section~\ref{sec:frozen_bias} analyzes frozen latent reuse on the target distribution and identifies the structural bias induced by subspace mismatch and target ambient noise. 
Section~\ref{sec:frozen_upper} complements this with an explicit comparator and a finite-sample bound for target-domain training with a frozen projector. 
Finally, Section~\ref{sec:mixed_training} studies mixed source-target training as an alternative to frozen reuse.

\subsection{Irreducible Error under Frozen Latent Reuse}
\label{sec:frozen_bias}
We first identify a structural limitation of frozen latent reuse: even if the downstream score model is arbitrarily expressive, a fixed source latent space may leave irreducible target-domain error.

Let \(V_1\in\mathbb R^{D\times d_1}\) be the source projector, with \(V_1^\top V_1=I_{d_1}\). 
In frozen reuse, \(V_1\) is kept fixed, and the target score model is restricted to the projected coordinates \(V_1^\top x\).
We consider score networks
\[
    s_{V_1,\theta}(x,t)
    =
    \frac{1}{h(t)}V_1 f_\theta(V_1^\top x,t)-\frac{1}{h(t)}x,
\]
where \(f_\theta:\mathbb R^{d_1}\times[t_0,T]\to\mathbb R^{d_1}\) belongs to the ReLU class
$
    \mathcal F_{\rm ReLU}^{(d_1)}
    :=
    \mathcal F_{\rm ReLU}^{(d_1)}(L,M,J,K,\kappa,\gamma,\gamma_t).
$
This class consists of ReLU networks with depth \(L\), width at most \(M\), sparsity at most \(J\), output bound \(K\), parameter bound \(\kappa\), and Lipschitz constants \(\gamma,\gamma_t\) in the spatial and time variables, respectively\footnote{Due to space constraints, we defer the explicit definition of $\mathcal F_{\rm ReLU}^{(d_1)}(L,M,J,K,\kappa,\gamma,\gamma_t)$ to the appendix.}. 
The frozen ReLU class is
\[
    \mathcal S_{\rm freeze}(V_1)
    :=
    \{s_{V_1,\theta}: f_\theta\in\mathcal F_{\rm ReLU}^{(d_1)}\},
\]
and its oracle risk is
\[
    \mathcal B_{\rm freeze}(V_1)
    :=
    \inf_{s\in\mathcal S_{\rm freeze}(V_1)}\mathcal L_2(s),
    \quad
    \mathcal L_2(s)
    :=
    \frac{1}{T-t_0}\int_{t_0}^T
    \E_{X_t\sim p_{2,t}}
    \|s(X_t,t)-\nabla\log p_{2,t}(X_t)\|_2^2\,\rmd t .
\]

To isolate the geometric limitation of the frozen projector from ReLU approximation error, we define the enlarged structural class
\[
    \mathcal S_{\rm str}(V_1)
    :=
    \left\{
    s_{V_1,f}(x,t)=\frac{1}{h(t)}V_1f(V_1^\top x,t)-\frac{1}{h(t)}x:
    f \text{ measurable}
    \right\},
\]
and
\[
    \mathcal B_{\rm str}(V_1)
    :=
    \inf_{s\in\mathcal S_{\rm str}(V_1)}\mathcal L_2(s).
\]
Since \(\mathcal S_{\rm freeze}(V_1)\subseteq\mathcal S_{\rm str}(V_1)\), we have
$
    \mathcal B_{\rm freeze}(V_1)\geqslant \mathcal B_{\rm str}(V_1).
$
Thus, a lower bound on \(\mathcal B_{\rm str}(V_1)\) is also a lower bound on the practical frozen ReLU oracle risk.
Let \(r=\min\{d_1,d_2\}\), and let \(0\leqslant\theta_1\leqslant\cdots\leqslant\theta_r\leqslant\pi/2\) be the principal angles between \(\operatorname{col}(V_1)\) and \(\operatorname{col}(A_2)\), defined by
\[
    \sigma_j(V_1^\top A_2)=\cos\theta_j,\qquad j=1,\ldots,r.
\]

We now establish a lower bound on the structural bias induced by a frozen projection.
\begin{theorem}
\label{prop:frozen_projection_lower_bound}
Suppose Assumption~\ref{asm:noisy_low_dim_shift} holds for \(p_2\). 
Let
\[
    g_2(y,t):=h(t)\nabla\log p_{2,t}^{\rm LD}(y)+y.
\]
For \(Y_t:=A_2^\top X_t\), define
\[
    M_{g,2}(t):=\E[g_2(Y_t,t)g_2(Y_t,t)^\top],
    \qquad
    \mu_2(t):=\lambda_{\min}(M_{g,2}(t)).
\]
Then,
\[
\begin{aligned}
    \mathcal B_{\rm str}(V_1)
    \geqslant
    \frac{1}{T-t_0}\int_{t_0}^T
    \left[
    \frac{\mu_2(t)}{h^2(t)}
    \left(d_2-\sum_{j=1}^r\cos^2\theta_j\right)
    +
    \frac{\sigma_2^4\alpha^4(t)}
    {h^2(t)\widetilde h_2(t)}
    \left(D-d_1-d_2+\sum_{j=1}^r\cos^2\theta_j\right)
    \right]\rmd t.
\end{aligned}
\]
Consequently, the same lower bound holds for \(\mathcal B_{\rm freeze}(V_1)\).
\end{theorem}

The first term captures target signal directions missed by the frozen source subspace, and it vanishes when 
\(\operatorname{col}(A_2)\subseteq \operatorname{col}(V_1)\). 
The second term captures target ambient noise outside the frozen output subspace, amplified by the diffusion time scale. 
Thus, the lower bound identifies a genuine limit of latent reuse: even an arbitrarily expressive core map cannot recover target score components that are absent from the frozen coordinates.

\subsection{Achievability and Statistical Error under Frozen Reuse}
\label{sec:frozen_upper}

The preceding lower bound identifies regimes in which frozen reuse incurs intrinsic error. 
We next complement this result by studying what can be achieved when the frozen subspace is sufficiently aligned with the target distribution. 
To this end, we introduce an explicit frozen-coordinate comparator and derive the corresponding statistical guarantee.

To establish the bounds, we require mild regularity of the target latent distribution and its induced score map.
\begin{assumption}
\label{asm:p2_decay}
The target latent density \(p_{2,z}>0\) is twice continuously differentiable. 
Moreover, there exist constants \(B_2,C_{1,2},C_{2,2}>0\) such that, for all \(\|z\|_2\geqslant B_2\),
it holds that $$
    p_{2,z}(z)
    \leqslant
    (2\pi)^{-d_2/2}C_{1,2}
    \exp\{-C_{2,2}\|z\|_2^2/2\}.$$
\end{assumption}

\begin{assumption}
\label{asm:s_paral_Lip}
The target core score map
$
    f_2^*(z,t)
    :=
    \widetilde h_2(t)\nabla\log p_{2,t}^{\rm LD}(z)+z
$
is uniformly \(\beta_2\)-Lipschitz in \(z\in\mathbb R^{d_2}\) for all \(t\in[0,T]\).
\end{assumption}
These assumptions are standard in theoretical analyses of diffusion models~\cite{Chen2023_scorelowdim,block2020generative}. 
They ensure that the target latent score is regular enough to be approximated by the frozen-coordinate comparator introduced below.
Let
\[    
B:=V_1^\top A_2,    \qquad    g_2(y,t):=h(t)\nabla\log p_{2,t}^{\rm LD}(y)+y .
\]
We use the explicit frozen-coordinate comparator
\[    
f_{\rm comp}(z,t):=B g_2(B^\dagger z,t),
\]
where \(B^\dagger\) is the Moore--Penrose pseudoinverse. 
This comparator reconstructs a target latent coordinate from the frozen coordinate \(z=V_1^\top x\) through \(B^\dagger z\), applies the target latent score map \(g_2\), and maps it back to the frozen coordinates through \(B\). 
The induced frozen score model is
\[    
s_{\rm comp}(x,t)    :=    \frac{1}{h(t)}V_1f_{\rm comp}(V_1^\top x,t)-\frac{1}{h(t)}x,
\]
and we denote $\mathcal U_{\rm comp}(V_1):=\mathcal L_2(s_{\rm comp}).$
Since $s_{\rm comp}\in\mathcal S_{\rm str}(V_1)$, we have $\mathcal B_{\rm str}(V_1)\leqslant \mathcal U_{\rm comp}(V_1).$
It remains to bound the comparator risk \(\mathcal U_{\rm comp}(V_1)\). 
The next result expresses this bound in terms of the principal angles between the frozen source subspace and the target signal subspace, together with the target ambient noise.
\begin{proposition}\label{prop:frozen_projection_upper_bound}
Suppose Assumptions~\ref{asm:noisy_low_dim_shift}, \ref{asm:p2_decay}, and~\ref{asm:s_paral_Lip} hold for \(p_2\). 
Let \(\theta_1,\ldots,\theta_r\) be the principal angles between \(\operatorname{col}(V_1)\) and \(\operatorname{col}(A_2)\), where \(r=\min\{d_1,d_2\}\). 
Let \(L_g(t)\) be a Lipschitz constant of \(g_2(\cdot,t)\), and define
\[
    \lambda_{g,2}^{\max}(t)
    :=
    \lambda_{\max}\!\left(
    \E[g_2(Y_t,t)g_2(Y_t,t)^\top]
    \right),
    \qquad 
    M_{Y,2}(t):=\E[Y_tY_t^\top],
    \qquad
    Y_t:=A_2^\top X_t .
\]

\emph{If \(d_1\geqslant d_2\)} and \(B=V_1^\top A_2\) has full column rank, it holds that
\[
\begin{aligned}
\mathcal U_{\rm comp}(V_1)
&\leqslant
\frac{1}{T-t_0}\int_{t_0}^T
\Bigg[
\frac{\lambda_{g,2}^{\max}(t)}{h^2(t)}
\sum_{j=1}^{d_2}\sin^2\theta_j
+
\frac{2\|B\|_{\rm op}^2\widetilde h_2(t)L_g^2(t)}{h^2(t)}
\sum_{j=1}^{d_2}\tan^2\theta_j
\\
&\qquad +
\frac{\sigma_2^4\alpha^4(t)}
{h^2(t)\widetilde h_2(t)}
\left(
D+d_1-d_2-\sum_{j=1}^{d_2}\cos^2\theta_j
\right)
\Bigg]\rmd t .
\end{aligned}
\]

\emph{If \(d_1<d_2\)  and \(B=V_1^\top A_2\) has full row  rank,} it holds that
\[
\begin{aligned}
\mathcal U_{\rm comp}(V_1)
& \leqslant
\frac{1}{T-t_0}\int_{t_0}^T
\Bigg[
\frac{\lambda_{g,2}^{\max}(t)}{h^2(t)}
\left(d_2-\sum_{j=1}^{d_1}\cos^2\theta_j\right)
+
\frac{2\|B\|_{\rm op}^2L_g^2(t)}{h^2(t)}
\operatorname{Tr}\!\left((I_{d_2}-B^\dagger B)M_{Y,2}(t)\right)
\\
&\qquad +
\frac{2\|B\|_{\rm op}^2\widetilde h_2(t)L_g^2(t)}{h^2(t)}
\sum_{j=1}^{d_1}\tan^2\theta_j
+
\frac{\sigma_2^4\alpha^4(t)}
{h^2(t)\widetilde h_2(t)}
\left(
D+d_1-d_2-\sum_{j=1}^{d_1}\cos^2\theta_j
\right)
\Bigg]\rmd t .
\end{aligned}
\]
\end{proposition}
{When \(d_1\geqslant d_2\), the frozen space has sufficient dimension to represent the target subspace, and the main difficulty is stability, as reflected by \(\tan^2\theta_j\). 
When \(d_1<d_2\), the frozen space is dimensionally insufficient, leading to the additional information-loss term $\operatorname{Tr}\left((I_{d_2}-B^\dagger B)M_{Y,2}(t)\right)$ in addition to the output-side residual \(d_2-\sum_{j=1}^{d_1}\cos^2\theta_j\). 
Thus, frozen reuse can fail either because the reused space is poorly aligned with the target subspace, or because it is too small to contain the target-relevant directions.}

We now connect the measurable comparator bound to the practical ReLU estimator.
The argument has two steps. First, the frozen ReLU oracle risk is controlled by the measurable comparator risk plus the cost of approximating the comparator within the ReLU class. Second, the empirical denoising-risk minimizer is controlled by this frozen oracle risk up to a finite-sample error.

Recall that
$
\mathcal B_{\rm freeze}(V_1)
:=
\inf_{s\in\mathcal S_{\rm freeze}(V_1)}
\mathcal L_2(s)
$
is the population oracle risk over the frozen ReLU class. To pass from the measurable comparator to this class, 
define
\[
\mathcal A_{\rm comp}(f_\theta;V_1)
:=
\frac{1}{T-t_0}
\int_{t_0}^T
\frac{1}{h^2(t)}
\E_{X_t\sim p_{2,t}}
\left[
\|f_\theta(V_1^\top X_t,t)-f_{\rm comp}(V_1^\top X_t,t)\|_2^2
\right]
\,\rmd t .
\]
This term measures the approximation error incurred by replacing the measurable comparator \(f_{\rm comp}\) with a ReLU network on the frozen-coordinate distribution.
\begin{lemma}[Comparator-to-ReLU transfer]
\label{lem:comparator_transfer}
Let
\begin{align*}
s_{\rm comp}(x,t):=\frac{1}{h(t)}V_1f_{\rm comp}(V_1^\top x,t)-\frac{1}{h(t)}x,
\end{align*}
and let \(\mathcal U_{\rm comp}(V_1):=\mathcal L_2(s_{\rm comp})\). Then, for any \(\eta>0\),
\begin{align*}
\mathcal B_{\rm freeze}(V_1)\leqslant (1+\eta)\,\mathcal U_{\rm comp}(V_1)+\left(1+\frac{1}{\eta}\right)\inf_{f_\theta\in\mathcal F_{\rm ReLU}^{(d_1)}}\mathcal A_{\rm comp}(f_\theta;V_1).
\end{align*}
\end{lemma}

By Lemma~\ref{lem:comparator_transfer}, for any \(\eta>0\),
\begin{align}
\label{eq:Bfreeze}
\mathcal B_{\rm freeze}(V_1)
\leq
(1+\eta)\,\mathcal U_{\rm comp}(V_1)
+
\left(1+\frac{1}{\eta}\right)
\inf_{f_\theta\in\mathcal F_{\rm ReLU}^{(d_1)}}
\mathcal A_{\rm comp}(f_\theta;V_1).
\end{align}
Thus, the frozen ReLU oracle risk is controlled by the geometry-dependent comparator error and the ReLU approximation error for \(f_{\rm comp}\).

We next pass from the frozen oracle risk to finite-sample training. Let \(x_1,\ldots,x_{N_2}\) be i.i.d. samples from \(p_2\). For any score network \(s\), define the target-domain denoising loss
\[
\ell_2(x;s)
:=
\frac{1}{T-t_0}
\int_{t_0}^T
\E_{X_t\sim p_{2,t}(\cdot\mid X_0=x)}
\left[
\left\|
\nabla_{X_t}\log\phi_t(X_t\mid x)-s(X_t,t)
\right\|_2^2
\right]
\,\rmd t ,
\]
where \(\phi_t(\cdot\mid x)\) denotes the Gaussian transition density of \(X_t\) conditional on \(X_0=x\). Define
\[
\mathcal R_2(s):=\E_{X_0\sim p_2}[\ell_2(X_0;s)],
\qquad
\widehat{\mathcal R}_2(s):=\frac{1}{N_2}\sum_{i=1}^{N_2}\ell_2(x_i;s),
\]
and let
\[
\widehat s_{\rm freeze}
\in
\argmin_{s\in\mathcal S_{\rm freeze}(V_1)}
\widehat{\mathcal R}_2(s).
\]

Although \(\widehat s_{\rm freeze}\) is trained by minimizing the empirical denoising risk, our target quantity is the population score error \(\mathcal L_2\). The denoising score-matching identity implies that there exists a constant \(E_2\), independent of \(s\), such that
\[
\mathcal R_2(s)=\mathcal L_2(s)+E_2 .
\]
The identity between the denoising objective and score matching follows from the denoising score matching argument of Vincent~\cite{vincent2011connection}.
Therefore, population minimization of \(\mathcal R_2\) is equivalent to population minimization of \(\mathcal L_2\). Moreover, if
\[
\mathrm{StatErr}_{N_2}
:=
2\sup_{s\in\mathcal S_{\rm freeze}(V_1)}
\left|
\widehat{\mathcal R}_2(s)-\mathcal R_2(s)
\right|,
\]
then the empirical minimizer satisfies the oracle inequality
\[
\mathcal L_2(\widehat s_{\rm freeze})
\leqslant
\mathcal B_{\rm freeze}(V_1)
+
\mathrm{StatErr}_{N_2}.
\]
Combining this inequality with the comparator decomposition in~\eqref{eq:Bfreeze} gives the following bound.
\begin{theorem}\label{thm:frozen_projector_statistical_bound}
Suppose Assumptions~\ref{asm:noisy_low_dim_shift}, \ref{asm:p2_decay}, and~\ref{asm:s_paral_Lip} hold for \(p_2\). 
Let \(V_1\in\mathbb R^{D\times d_1}\) be fixed independently of the target samples, {with \(V_1^\top V_1=I_{d_1}\)}. 
For \(\delta\in(0,1)\), set
\[
R_z=C_z(1+\sigma_2)
\left(\sqrt{d_2}+\sqrt{\log\frac{8N_2}{\delta}}\right),
\qquad
R_\perp=C_\perp\sigma_2
\left(\sqrt{D-d_2}+\sqrt{\log\frac{8N_2}{\delta}}\right),
\]
and define
\[
\mathcal T_2(R_z,R_\perp)
:=
\left\{
x\in\mathbb R^D:
\|A_2^\top x\|_2\le R_z,\ 
\|P_2^\perp x\|_2\le R_\perp
\right\}.
\]
For \(\iota>0\), let
\[
\mathcal N_2(\iota)
:=
\mathcal N
\left(
\iota,
\left\{
\ell_2(\cdot;s)\mathbf 1_{\mathcal T_2(R_z,R_\perp)}(\cdot):
s\in\mathcal S_{\rm freeze}(V_1)
\right\},
\|\cdot\|_\infty
\right).
\]
If the ReLU architecture is chosen so that
\[
\inf_{f_\theta\in\mathcal F_{\rm ReLU}^{(d_1)}}
\mathcal A_{\rm comp}(f_\theta;V_1)
\lesssim
\frac{(d_1+1)\epsilon^2}{t_0(T-t_0)},
\]
then, with probability at least \(1-\delta\), it holds that
\[
\mathcal L_2(\widehat s_{\rm freeze})
\leqslant
\widetilde{\mathcal O}
\left(
\frac{K^2+R_z^2+R_\perp^2}
{N_2t_0(T-t_0)\epsilon^2}
\left[
\log\mathcal N_2(\iota)+\log\frac{8}{\delta}
\right]
+\iota\right)
+
\mathcal O\!\left(\frac{(d_1+1)\epsilon^2}{t_0(T-t_0)}\right)
+
\mathcal O\!\left(\mathcal U_{\rm comp}(V_1)\right).
\]
Here, \(K\) is the uniform output bound of the ReLU network class, and \(C_z,C_\perp\) depend only on the target latent tail constants and absolute Gaussian concentration constants.

\end{theorem}

The bound decomposes the target-domain error into three terms: finite-sample estimation error, ReLU approximation error, and the geometry-dependent comparator error \(\mathcal U_{\rm comp}(V_1)\). The first term decreases as the number \(N_2\) of target samples increases, and the second term can be reduced by enlarging the ReLU class. In contrast, \(\mathcal U_{\rm comp}(V_1)\) is determined by the compatibility between the frozen representation and the target distribution, as quantified in Proposition~\ref{prop:frozen_projection_upper_bound}.

This reveals a fundamental limitation of latent reuse. If the frozen subspace is not well aligned with the target subspace, or if its dimension is too small to contain the target-relevant directions, then the geometric term~\(\mathcal U_{\rm comp}(V_1)\) may remain large. In this case, latent reuse can fail even with many target samples and an expressive downstream score network, because the relevant target information has already been lost or distorted by the frozen representation. 
This motivates a different strategy: rather than reusing a fixed source representation, we train on mixed source--target data so that the learned representation can better reflect the target-domain structure.

\subsection{Mixed Source-target Training}\label{sec:mixed_training}

In this section, we study this idea at the oracle level by comparing two fixed projectors: the frozen source projector \(V_1\) and a newly learned shared projector \(W_k\) of dimension \(k\).
Let
\[
    p_{\rm mix}:=\omega_1 p_1+\omega_2 p_2,
\]
where \(\omega_1,\omega_2>0\) and \(\omega_1+\omega_2=1\).
For a \(k\)-dimensional projector \(W\in\mathbb R^{D\times k}\) and weights \(c_1,c_2>0\), define the weighted signal residual
\[
    \Gamma_k(W)
    :=
    \omega_1c_1\|P_W^\perp A_1\|_F^2
    +
    \omega_2c_2\|P_W^\perp A_2\|_F^2\,.
\]
Here, \(c_i\) scales the contribution of the signal residual from component \(i\). Thus, \(\Gamma_k(W)\) measures the weighted amount of source and target signal lost outside \(\operatorname{col}(W)\).

\begin{proposition}
\label{prop:optimal_mixed_projector}
Let \(A_i\in\mathbb R^{D\times d_i}\) have orthonormal columns, and constants \(c_1, c_2>0\). 
For \(W\in\mathbb R^{D\times k}\) with \(W^\top W=I_k\), let
$
    \Gamma_k(W)
    =
    \sum_{i=1}^2\omega_i c_i\|P_W^\perp A_i\|_F^2\, .
$
Set
$
    M_{\rm mix}:=\sum_{i=1}^2\omega_i c_i A_iA_i^\top,
$
and let \(\lambda_1(M_{\rm mix})\geqslant\cdots\geqslant\lambda_D(M_{\rm mix})\geqslant 0\). 
Then, for any \(k\in\{\max(d_1,d_2),\ldots,D\}\), it holds that
\[
    \inf_{W:W^\top W=I_k}\Gamma_k(W)
    =
    \operatorname{Tr}(M_{\rm mix})-\sum_{j=1}^k\lambda_j(M_{\rm mix})
    =
    \sum_{j=k+1}^D\lambda_j(M_{\rm mix}).
\]
The infimum is attained by any \(W_k\) whose columns span a top-\(k\) eigenspace of \(M_{\rm mix}\). 
In particular, the optimal residual is nonincreasing in \(k\) and vanishes if
$
    \operatorname{span}(A_1)+\operatorname{span}(A_2)
    \subseteq
    \operatorname{span}(W_k).
$
\end{proposition}

Proposition~\ref{prop:optimal_mixed_projector} gives a simple geometric reason to retrain a shared latent space. 
If the frozen source projector \(V_1\) misses target directions, then \(W_k\) minimizes the joint source--target signal residual among all \(k\)-dimensional projectors. 
When \(W_k\) contains the union of the source and target signal subspaces, this signal-side residual is completely removed.

We next connect the geometric improvement of the mixed projector to the mixed score-oracle error. 
For a fixed orthonormal projector \(U\in\mathbb R^{D\times m}\), with \(U=V_1\) or \(U=W_k\), define
\[
    \mathcal B_{\rm mix}(U)
    :=
    \inf_{s\in\mathcal S_{\rm mix}(U)}
    \mathcal L_{\rm mix}(s),
\]
where
\[
 \mathcal S_{\rm mix}(U)
    :=
    \left\{
    s_{U,\theta}(x,t)
    =
    \frac{1}{h(t)}Uf_\theta(U^\top x,t)-\frac{1}{h(t)}x
    :
    f_\theta\in\mathcal F_{\rm ReLU}^{(m)}
    \right\}
\]
is the score class using the projected coordinates \(U^\top x\), and \(\mathcal L_{\rm mix}\) denotes the integrated score risk under \(p_{\rm mix}\):
\begin{align*}
\mathcal L_{\rm mix}(s):=\frac{1}{T-t_0}\int_{t_0}^T\E_{X_t\sim p_{{\rm mix},t}}\left[\|s(X_t,t)-\nabla\log p_{{\rm mix},t}(X_t)\|_2^2\right]\,\rmd t.
\end{align*}
The leading projector-dependent quantity is
\[
    \Gamma(U)
    :=
    \sum_{i=1}^2
    \omega_i\bar c_i
    \|P_U^\perp A_i\|_F^2\,, \quad \text{with}\quad \bar c_i
    :=
    \frac{1}{T-t_0}
    \int_{t_0}^T
    \frac{\lambda_{g,i}^{\max}(t)}{h^2(t)}\,\rmd t \,,
\]
where
\[
    \lambda_{g,i}^{\max}(t)
    :=
    \lambda_{\max}\!\left(M_{g,i}(t)\right),
    \qquad
    M_{g,i}(t)
    :=
    \E_{X_t\sim p_{i,t}}
    \left[
    g_i(A_i^\top X_t,t)g_i(A_i^\top X_t,t)^\top
    \right],
\]
with
\[
    g_i(y,t):=h(t)\nabla\log p_{i,t}^{\rm LD}(y)+y.
\]
Thus, \(\Gamma(U)\) measures the weighted source--target signal residual outside the shared latent space. 
With the natural weights \(c_i=\bar c_i\), the next result bounds the mixed oracle risk and uses \(W_k\) to minimize the leading signal-residual term.
\begin{theorem}
\label{thm:mixed_oracle_projector_comparison}
Suppose Assumption~\ref{asm:noisy_low_dim_shift} and~\ref{asm:s_paral_Lip} hold for both \(p_1\) and \(p_2\), together with the regularity assumptions required for the mixed comparators, whose precise statement is deferred to Appendix~\ref{app:mixed_training_proofs}. 
Let \(W_k\) be the optimizer in Proposition~\ref{prop:optimal_mixed_projector} with \(c_i=\bar c_i\). 
Then, for \(U\in\{V_1,W_k\}\) and any \(\eta>0\),
\[
\begin{aligned}
\mathcal B_{\rm mix}(U)
&\leqslant
(1+\eta)
\Bigg[
    \Gamma(U)
    +
    \sum_{i=1}^2\omega_i\bar n_i\operatorname{Tr}(P_U^\perp P_i^\perp)
    +
    2\sum_{i=1}^2\omega_i\mathfrak R_i(U)
    +
    2\mathfrak P(U)
\Bigg] \\
&\qquad +
\left(1+\frac{1}{\eta}\right)
\inf_{f_\theta\in\mathcal F_{\rm ReLU}^{(m)}}
\mathcal A_{\rm mix}(f_\theta;U).
\end{aligned}
\]
\end{theorem}

Theorem~\ref{thm:mixed_oracle_projector_comparison} separates the mixed oracle risk into a leading signal-residual term, an ambient-noise term, reconstruction $\mathfrak R_i(U)$ and posterior-compression terms $\mathfrak P(U)$. 
The auxiliary quantities \(\mathfrak R_i(U)\) and \(\mathfrak P(U)\) capture, respectively, the componentwise reconstruction error from the projected coordinates \(U^\top x\) and the loss of mixture-component information after projection. 
Their explicit forms, together with the full mixed oracle decomposition, are given in Appendix~\ref{app:mixed_training_proofs}.

The leading projector-dependent term in the upper bound of Theorem~\ref{thm:mixed_oracle_projector_comparison} is \(\Gamma(U)\). 
Since
$
    M_{\rm mix}=\sum_{i=1}^2\omega_i\bar c_i A_iA_i^\top,
$
Proposition~\ref{prop:optimal_mixed_projector} implies that the optimal \(k\)-dimensional mixed projector \(W_k\) satisfies
$
    \Gamma(W_k)
    =
    \sum_{j=k+1}^D\lambda_j(M_{\rm mix}).
$
Thus, compared with frozen reuse, the shared projector \(W_k\) reduces the leading signal residual by
\[
    \Gamma(V_1)-\Gamma(W_k)
    =
    \Gamma(V_1)-\sum_{j=k+1}^D\lambda_j(M_{\rm mix}).
\]
Thus, mixed training improves the oracle upper bound when this reduction dominates the corresponding changes in the ambient-noise, reconstruction, posterior-compression, and ReLU approximation terms. The following corollary illustrates the most favorable case, where the mixed projector contains both source and target signal subspaces.
\begin{corollary}
\label{cor:mixed_containment_regime}
Suppose $W_k$ satisfies
$
    \operatorname{span}(A_1)+\operatorname{span}(A_2)
    \subseteq
    \operatorname{span}(W_k).
$
Then, for any \(\eta>0\),
\[
\begin{aligned}
\mathcal B_{\rm mix}(W_k)
&\leqslant
(1+\eta)
\left[
    \sum_{i=1}^2\omega_i\bar n_i(D-k)
    +
    2\sum_{i=1}^2\omega_i\bar n_i(k-d_i)
    +
    2\mathfrak P(W_k)
\right] \\
&\qquad
+
\left(1+\frac{1}{\eta}\right)
\inf_{f_\theta\in\mathcal F_{\rm ReLU}^{(k)}}
\mathcal A_{\rm mix}(f_\theta;W_k)\,,
\end{aligned}
\]
where $ \bar n_i:=\frac{1}{T-t_0}\int_{t_0}^T\frac{\alpha^4(t)\sigma_i^4}{h^2(t)\tilde h_i(t)}\,\rmd t$. In particular, if \(V_1=A_1\), then
\[
    \Gamma(V_1)
    =
    \omega_2\bar c_2\|P_{A_1}^\perp A_2\|_F^2,
    \qquad
    \Gamma(W_k)=0.
\]
Thus, in the containment regime, mixed training removes the target signal-mismatch term that remains under source-aligned reuse.

\end{corollary}

The posterior-compression term \(\mathfrak P(W_k)\) does not necessarily vanish under the containment condition alone. It vanishes under the stronger condition that the projected coordinate \(W_k^\top X_t\) is sufficient for identifying the mixture component. For instance, this holds when \(W_k\) contains both signal subspaces and the two mixture components have identical ambient-noise distributions. In this case, mixed training eliminates the source--target signal-mismatch term, leaving only ambient-noise, reconstruction, and ReLU approximation contributions.

\section{Discussion}

Our analysis focuses on a noisy linear low-dimensional model, where the source and target distributions concentrate near possibly different subspaces. This setting makes the geometric obstruction to latent reuse explicit through principal angles and ambient-noise terms, but it does not capture the nonlinear latent representations often used in practice. A natural next step is to extend the theory to nonlinear low-dimensional structures, such as distributions supported near smooth manifolds, as studied for example in~\cite{azangulov2024convergence,yakovlev2025generalization,tang2024adaptivity}. In such settings, subspace mismatch would be replaced by tangent-space mismatch and curvature effects.

Another limitation is that our mixed-training analysis is carried out at the oracle level, treating the shared projector as fixed once chosen. This isolates the representation-level mechanism, but does not account for the finite-sample cost of learning the projector from mixed data. A more complete theory would combine representation learning error, score-estimation error, and sampling error in a single end-to-end guarantee. It would also be interesting to study adaptive mixtures, where the relative amount of source and target data is chosen according to the degree of distribution shift.

\clearpage
\bibliographystyle{amsalpha}
{\small \bibliography{bib}}
\newpage
\appendix
\section{Additional Key Definitions}

In this work, we consider the ReLU network $\mathcal F_{\rm ReLU}^{(d_1)}(L,M,J,K,\kappa,\gamma,\gamma_t)$ defined via
\begin{align*}
    \mathcal F_{\rm ReLU}^{(d_1)}(L,M,J,K,\kappa,\gamma,\gamma_t)=\bigg\{&f(z,t)=W_L\sigma(\cdots\sigma(W_1[z^\top,t]^\top+b_1)\cdots)+b_L:\\
    &\text{network width bounded by }M,\,\sup_{z,t}\|f(z,t)\|_2\leqslant K,\\
    &\max\{\|b_i\|_\infty,\|W_i\|_\infty\}\leqslant \kappa\,\text{ for }\,i=1,\cdots, L,\\
    &\sum_{i=1}^L(\|W_i\|_0+\|b_i\|_0)\leqslant J,\\
    &\|f(z_1,t)-f(z_2,t)\|_2\leqslant\gamma\|z_1-z_2\|_2\,\text{ for any }\,t\in[0,T],\\
    &\|f(z,t_1)-f(z,t_2)\|_2\leqslant \gamma_t|t_1-t_2|\,\text{ for any }\, z\bigg\}\,.
\end{align*}

\section{Proof of Auxiliary Lemmas}

\begin{proof}[Proof of Lemma~\ref{lem:noisy_score_decomp}]
    The proof follows the exact orthogonal decomposition framework as Lemma 1 in \cite{Chen2023_scorelowdim}, with one crucial algebraic modification regarding the transition kernel.
    
    Under our Assumption \ref{asm:noisy_low_dim_shift}, the initial data incorporates ambient noise $X_0 = A_i z + \epsilon_i$, where $\epsilon_i \sim \mathcal{N}(0, \sigma_i^2 I_D)$. By the convolution of independent Gaussian distributions, the conditional transition kernel of the forward diffusion process becomes $X_t | z \sim \mathcal{N}(\alpha(t)A_i z, \tilde{h}_i(t) I_D)$, where the effective variance is strictly $\tilde{h}_i(t) = \alpha^2(t)\sigma_i^2 + h(t)$. 
    
    Substituting this effective variance $\tilde{h}_i(t)$ for the pure diffusion variance $h(t)$ in the spatial Pythagorean decomposition (i.e., $\|x - \alpha(t)A_i z\|_2^2 = \|A_i^\top x - \alpha(t)z\|_2^2 + \|(I_D - A_i A_i^\top)x\|_2^2$) directly yields the stated factorization of the marginal density and the corresponding score function.
\end{proof}

\begin{proof}[Proof of Lemma~\ref{lem:comparator_transfer}]
Fix any \(f_\theta\in\mathcal F_{\rm ReLU}^{(d_1)}\), and let
\begin{align*}
s_{V_1,\theta}(x,t):=\frac{1}{h(t)}V_1f_\theta(V_1^\top x,t)-\frac{1}{h(t)}x.
\end{align*}
For any \(x\in\mathbb R^D\) and \(t\in[t_0,T]\), we decompose
\begin{align*}
s_{V_1,\theta}(x,t)-\nabla\log p_{2,t}(x)=\left[s_{\rm comp}(x,t)-\nabla\log p_{2,t}(x)\right]+\frac{1}{h(t)}V_1\left[f_\theta(V_1^\top x,t)-f_{\rm comp}(V_1^\top x,t)\right].
\end{align*}
Using \(\|a+b\|_2^2\leqslant(1+\eta)\|a\|_2^2+(1+1/\eta)\|b\|_2^2\), and using \(V_1^\top V_1=I_{d_1}\), we obtain
\begin{align*}
\|s_{V_1,\theta}(x,t)-\nabla\log p_{2,t}(x)\|_2^2&\leqslant (1+\eta)\|s_{\rm comp}(x,t)-\nabla\log p_{2,t}(x)\|_2^2\\
&\quad+\left(1+\frac{1}{\eta}\right)\frac{1}{h^2(t)}\|f_\theta(V_1^\top x,t)-f_{\rm comp}(V_1^\top x,t)\|_2^2.
\end{align*}
Taking expectation over \(X_t\sim p_{2,t}\), integrating over \(t\in[t_0,T]\), and dividing by \(T-t_0\), gives
\begin{align*}
\mathcal L_2(s_{V_1,\theta})\leqslant (1+\eta)\mathcal U_{\rm comp}(V_1)+\left(1+\frac{1}{\eta}\right)\mathcal A_{\rm comp}(f_\theta;V_1).
\end{align*}
Finally, taking the infimum over \(f_\theta\in\mathcal F_{\rm ReLU}^{(d_1)}\) yields
\begin{align*}
\mathcal B_{\rm freeze}(V_1)=\inf_{f_\theta\in\mathcal F_{\rm ReLU}^{(d_1)}}\mathcal L_2(s_{V_1,\theta})\leqslant (1+\eta)\mathcal U_{\rm comp}(V_1)+\left(1+\frac{1}{\eta}\right)\inf_{f_\theta\in\mathcal F_{\rm ReLU}^{(d_1)}}\mathcal A_{\rm comp}(f_\theta;V_1).
\end{align*}
\end{proof}

\begin{lemma}
\label{lem:relu_box_approx}
Let \(F:\mathbb R^m\times[t_0,T]\to\mathbb R^m\) satisfy, on \([-R,R]^m\times[t_0,T]\),
\begin{align*}
\|F(z,t)-F(z',t')\|_2\leqslant L_z\|z-z'\|_2+L_t|t-t'|.
\end{align*}
Assume also
\begin{align*}
\sup_{(z,t)\in[-R,R]^m\times[t_0,T]}\|F(z,t)\|_2\leqslant K_0.
\end{align*}
Then for every \(\epsilon>0\), there exist architecture parameters
\begin{align*}
L,M,J,K,\kappa,\gamma,\gamma_t
\end{align*}
chosen by the constructive approximation scheme in the proof of Theorem~1 of~\cite{Chen2023_scorelowdim}, with the replacements
\begin{align*}
d\mapsto m, \qquad 1+\beta\mapsto L_z,\qquad \tau\mapsto L_t,\qquad R\mapsto R,
\end{align*}
and a ReLU network
\begin{align*}
f_\theta\in\mathcal F_{\rm ReLU}^{(m)}(L,M,J,K,\kappa,\gamma,\gamma_t)
\end{align*}
such that
\begin{align*}
\sup_{(z,t)\in[-R,R]^m\times[t_0,T]}\|f_\theta(z,t)-F(z,t)\|_\infty\leqslant\epsilon.
\end{align*}
Moreover,
\begin{align*}
\sup_{z\in\mathbb R^m,\,t\in[t_0,T]}\|f_\theta(z,t)\|_2\leqslant CK_0.
\end{align*}
\end{lemma}

\section{Proof of Main Results}

\subsection{Proof of Section~\ref{sec:frozen_bias}}

\begin{proof}[Proof of Theorem~\ref{prop:frozen_projection_lower_bound}]
Fix \(t\in[t_0,T]\). By the definition of the frozen-projection hypothesis class and the definition of \(G_{2,t}\), for any measurable \(f\),
\begin{align*}
    s_{V_1,f}(X_t,t)-\nabla\log p_{2,t}(X_t)=\frac{1}{h(t)}\left[V_1f(V_1^\top X_t,t)-G_{2,t}(X_t)\right].
\end{align*}
Therefore,
\begin{align*}
    b_{\rm str}(t;V_1)=\frac{1}{h^2(t)}\inf_{f_t \text{ measurable}}\E\left[\left\|V_1f_t(V_1^\top X_t)-G_{2,t}(X_t)\right\|_2^2\right].
\end{align*}
Since \(V_1f_t(V_1^\top X_t)\in\operatorname{col}(V_1)\), we decompose
\begin{align*}
    G_{2,t}(X_t)=P_{V_1}G_{2,t}(X_t)+P_{V_1}^\perp G_{2,t}(X_t).
\end{align*}
Then \(V_1f_t(V_1^\top X_t)-P_{V_1}G_{2,t}(X_t)\in\operatorname{col}(V_1)\), whereas \(P_{V_1}^\perp G_{2,t}(X_t)\in\operatorname{col}(V_1)^\perp\). Hence the two terms are orthogonal and
\begin{align*}
    \left\|V_1f_t(V_1^\top X_t)-G_{2,t}(X_t)\right\|_2^2=\left\|V_1f_t(V_1^\top X_t)-P_{V_1}G_{2,t}(X_t)\right\|_2^2+\left\|P_{V_1}^\perp G_{2,t}(X_t)\right\|_2^2.
\end{align*}
Taking the infimum over \(f_t\), the first term is nonnegative. Therefore,
\begin{align}
\label{eq:frozen_lower_start_general}
    b_{\rm str}(t;V_1)\geqslant \frac{1}{h^2(t)}\E\left[\left\|P_{V_1}^\perp G_{2,t}(X_t)\right\|_2^2\right].
\end{align}
Recall that
\begin{align*}
    G_{2,t}(X_t)=A_2g_2(A_2^\top X_t,t)+\rho_2(t)P_2^\perp X_t.
\end{align*}
Let
\begin{align*}
    Y_t:=A_2^\top X_t,\qquad U_t:=P_2^\perp X_t.
\end{align*}
Then
\begin{align*}
    P_{V_1}^\perp G_{2,t}(X_t)=P_{V_1}^\perp A_2g_2(Y_t,t)+\rho_2(t)P_{V_1}^\perp U_t.
\end{align*}
By the same orthogonal factorization argument used in the noisy score decomposition lemma, $Y_t=A_2^\top X_t$ and $U_t=P_2^\perp X_t$ are independent under $p_{2,t}$. Since $g_2(Y_t,t)$ is measurable with respect to $Y_t$, it is independent of $U_t$. Note that $U_t\sim \mathcal N(0,\tilde h_2(t)P_2^\perp)$, and $\E[U_t]=0$.
Therefore, the cross term vanishes:
\begin{align*}
    \E\left[g_2(Y_t,t)^\top A_2^\top P_{V_1}^\perp U_t\right]=\E[g_2(Y_t,t)]^\top A_2^\top P_{V_1}^\perp\E[U_t]=0.
\end{align*}
Consequently,
\begin{align}
\label{eq:frozen_orth_expand_general}
    \E\left[\left\|P_{V_1}^\perp G_{2,t}(X_t)\right\|_2^2\right]=\E\left[\left\|P_{V_1}^\perp A_2g_2(Y_t,t)\right\|_2^2\right]+\rho_2^2(t)\E\left[\left\|P_{V_1}^\perp U_t\right\|_2^2\right].
\end{align}
We first lower bound the first term. Since
\begin{align*}
    M_{g,2}(t)=\E\left[g_2(Y_t,t)g_2(Y_t,t)^\top\right],
\end{align*}
we have
\begin{align*}
    \E\left[\left\|P_{V_1}^\perp A_2g_2(Y_t,t)\right\|_2^2\right]=\operatorname{Tr}\left(A_2^\top P_{V_1}^\perp A_2M_{g,2}(t)\right).
\end{align*}
The matrix \(A_2^\top P_{V_1}^\perp A_2\) is positive semidefinite. Since \(M_{g,2}(t)\succeq \mu_2(t)I_{d_2}\), we obtain
\begin{align*}
    \operatorname{Tr}\left(A_2^\top P_{V_1}^\perp A_2M_{g,2}(t)\right)\geqslant \mu_2(t)\operatorname{Tr}\left(A_2^\top P_{V_1}^\perp A_2\right)=\mu_2(t)\|P_{V_1}^\perp A_2\|_F^2.
\end{align*}
Thus,
\begin{align}
\label{eq:frozen_first_lower_general}
    \E\left[\left\|P_{V_1}^\perp A_2g_2(Y_t,t)\right\|_2^2\right]\geqslant \mu_2(t)\|P_{V_1}^\perp A_2\|_F^2.
\end{align}
Next, since \(\E[U_tU_t^\top]=\tilde h_2(t)P_2^\perp\),
\begin{align}
\label{eq:frozen_noise_term_general}
    \E\left[\left\|P_{V_1}^\perp U_t\right\|_2^2\right]=\operatorname{Tr}\left(P_{V_1}^\perp\E[U_tU_t^\top]P_{V_1}^\perp\right)=\tilde h_2(t)\operatorname{Tr}\left(P_{V_1}^\perp P_2^\perp\right).
\end{align}
Using
\begin{align*}
    \rho_2(t)=\frac{\alpha^2(t)\sigma_2^2}{\tilde h_2(t)},
\end{align*}
we get
\begin{align}
\label{eq:frozen_noise_scaled_general}
    \rho_2^2(t)\E\left[\left\|P_{V_1}^\perp U_t\right\|_2^2\right]=\frac{\alpha^4(t)\sigma_2^4}{\tilde h_2(t)}\operatorname{Tr}\left(P_{V_1}^\perp P_2^\perp\right).
\end{align}
Combining \eqref{eq:frozen_lower_start_general}, \eqref{eq:frozen_orth_expand_general}, \eqref{eq:frozen_first_lower_general}, and \eqref{eq:frozen_noise_scaled_general}, we obtain
\begin{align*}
    b_{\rm str}(t;V_1)\geqslant \frac{\mu_2(t)}{h^2(t)}\|P_{V_1}^\perp A_2\|_F^2+\frac{\sigma_2^4\alpha^4(t)}{h^2(t)\tilde h_2(t)}\operatorname{Tr}\left(P_{V_1}^\perp P_2^\perp\right).
\end{align*}
For any $s_{V_1,f}\in \mathcal S_{\rm str}(V_1)$, we have
\begin{align*}
    r_{2,t}(s_{V_1,f})\geqslant b_{\rm str}(t;V_1)\,.
\end{align*}
Therefore, 
\begin{align*}
    \mathcal B_{\rm str}(V_1)&=\inf_{f\text{ measurable}}\mathcal L_2(s_{V_1,f})\\
    &=\inf_{s_{V_1,f}\in\mathcal S_{\rm str}(V_1)}\dfrac{1}{T-t_0}\int_{t_0}^Tr_{2,t}(s_{V_1,f})\,\rmd t\\
    &\geqslant \dfrac{1}{T-t_0}\int_{t_0}^Tb_{\rm str}(t;V_1)\,\rmd t\,.
\end{align*}

It remains to express the two geometric quantities using principal angles. By the definition of principal angles,
\begin{align*}
    \|V_1^\top A_2\|_F^2=\sum_{j=1}^r\cos^2\theta_j.
\end{align*}
For the first term,
\begin{align*}
    \|P_{V_1}^\perp A_2\|_F^2=\operatorname{Tr}\left(A_2^\top P_{V_1}^\perp A_2\right)=\operatorname{Tr}(A_2^\top A_2)-\operatorname{Tr}(A_2^\top P_{V_1}A_2)=d_2-\|V_1^\top A_2\|_F^2=d_2-\sum_{j=1}^r\cos^2\theta_j.
\end{align*}
For the second term,
\begin{align*}
    \operatorname{Tr}\left(P_{V_1}^\perp P_2^\perp\right)=\operatorname{Tr}\left((I_D-P_{V_1})(I_D-P_2)\right)=D-d_1-d_2+\operatorname{Tr}(P_{V_1}P_2)=D-d_1-d_2+\|V_1^\top A_2\|_F^2.
\end{align*}
Therefore,
\begin{align*}
    \operatorname{Tr}\left(P_{V_1}^\perp P_2^\perp\right)=D-d_1-d_2+\sum_{j=1}^r\cos^2\theta_j.
\end{align*}
Substituting these two identities into the integrated lower bound proves the principal-angle form.

\end{proof}

\subsection{Proof of Section~\ref{sec:frozen_upper}}

\begin{proof}[Proof of Proposition~\ref{prop:frozen_projection_upper_bound}]
Fix \(t\in[t_0,T]\). By the definition of \(G_{2,t}\), for any measurable \(f:\mathbb R^{d_1}\times[t_0,T]\to\mathbb R^{d_1}\),
\begin{align*}
    s_{V_1,f}(X_t,t)-\nabla\log p_{2,t}(X_t)=\frac{1}{h(t)}\left[V_1f(V_1^\top X_t,t)-G_{2,t}(X_t)\right].
\end{align*}
Therefore,
\begin{align*}
    r_{2,t}(s_{V_1,f})=\frac{1}{h^2(t)}\E\left[\left\|V_1f(V_1^\top X_t,t)-G_{2,t}(X_t)\right\|_2^2\right].
\end{align*}
Decompose \(G_{2,t}(X_t)\) into its components in \(\operatorname{col}(V_1)\) and \(\operatorname{col}(V_1)^\perp\):
\begin{align*}
    G_{2,t}(X_t)=P_{V_1}G_{2,t}(X_t)+P_{V_1}^\perp G_{2,t}(X_t).
\end{align*}
Since \(V_1f(V_1^\top X_t,t)-P_{V_1}G_{2,t}(X_t)\in\operatorname{col}(V_1)\) and \(P_{V_1}^\perp G_{2,t}(X_t)\in\operatorname{col}(V_1)^\perp\), the two terms are orthogonal. Hence
\begin{align*}
    r_{2,t}(s_{V_1,f})=\frac{1}{h^2(t)}\E\left[\left\|f(V_1^\top X_t,t)-V_1^\top G_{2,t}(X_t)\right\|_2^2\right]+\frac{1}{h^2(t)}\E\left[\left\|P_{V_1}^\perp G_{2,t}(X_t)\right\|_2^2\right].
\end{align*}
Consider the measurable predictor
\begin{align*}
    f_{\rm comp}(z,t):=Bg_2(B^\dagger z,t).
\end{align*}
Define 
\begin{align*}
    s_{\rm comp}(x,t)=\dfrac{1}{h(t)}V_1f_{\rm comp}(V_1^\top x,t)-\dfrac{1}{h(t)}x,\qquad \mathcal U_{\rm comp}(V_1):=\mathcal L_2(s_{\rm comp})\,.
\end{align*}
Since \(f_{\rm comp}\) is measurable, \(s_{\rm comp}\in\mathcal S_{\rm str}(V_1)\). Hence
\begin{align*}
    \mathcal B_{\rm str}(V_1)\leqslant \mathcal U_{\rm comp}(V_1)\,.
\end{align*}
To bound $\mathcal U_{\rm comp}(V_1)$, we have
\begin{align*}
    r_{2,t}(s_{\rm comp})=\dfrac{1}{h^2(t)}\underbrace{\E\left[\left\|f_{\rm comp}(V_1^\top X_t,t)-V_1^\top G_{2,t}(X_t)\right\|_2^2\right]}_{\mathcal I_t}+\dfrac{1}{h^2(t)}\underbrace{\E\left[\left\|P_{V_1}^\perp G_{2,t}(X_t)\right\|_2^2\right]}_{\mathcal O_t}\,.
\end{align*}

We first bound \(\mathcal O_t\). Recall that
\begin{align*}
    G_{2,t}(X_t)=A_2g_2(Y_t,t)+\rho_2(t)U_t,\qquad Y_t:=A_2^\top X_t,\quad U_t:=P_2^\perp X_t.
\end{align*}
By the same orthogonal factorization argument used in the noisy score decomposition lemma, \(Y_t\) and \(U_t\) are independent, \(U_t\sim\mathcal N(0,\tilde h_2(t)P_2^\perp)\), and \(\E[U_t]=0\).
Therefore,
\begin{align*}
    \mathcal O_t=\E\left[\left\|P_{V_1}^\perp A_2g_2(Y_t,t)\right\|_2^2\right]+\rho_2^2(t)\E\left[\left\|P_{V_1}^\perp U_t\right\|_2^2\right].
\end{align*}
Using \(M_{g,2}(t)\preceq \lambda_{g,2}^{\max}(t)I_{d_2}\), we have
\begin{align*}
    \E\left[\left\|P_{V_1}^\perp A_2g_2(Y_t,t)\right\|_2^2\right]=\operatorname{Tr}(A_2^\top P_{V_1}^\perp A_2 M_{g,2}(t))\leqslant \lambda_{g,2}^{\max}(t)\operatorname{Tr}(A_2^\top P_{V_1}^\perp A_2)=\lambda_{g,2}^{\max}(t)\|P_{V_1}^\perp A_2\|_F^2.
\end{align*}
Also,
\begin{align*}
    \E\left[\left\|P_{V_1}^\perp U_t\right\|_2^2\right]=\operatorname{Tr}\left(P_{V_1}^\perp\E[U_tU_t^\top]P_{V_1}^\perp\right)=\tilde h_2(t)\operatorname{Tr}\left(P_{V_1}^\perp P_2^\perp\right).
\end{align*}
Thus
\begin{align}
    \label{eq:orthogonal_upper_bound_general}
    \mathcal O_t\leqslant \lambda_{g,2}^{\max}(t)\|P_{V_1}^\perp A_2\|_F^2+\rho_2^2(t)\tilde h_2(t)\operatorname{Tr}\left(P_{V_1}^\perp P_2^\perp\right).
\end{align}

Next, we bound \(\mathcal I_t\). Recall
\begin{align*}
    Z_t:=V_1^\top X_t,\qquad Q_t:=V_1^\top G_{2,t}(X_t).
\end{align*}
Since \(X_t=A_2Y_t+U_t\), we have
\begin{align*}
    Z_t=V_1^\top X_t=V_1^\top A_2Y_t+V_1^\top U_t=BY_t+V_1^\top U_t.
\end{align*}
Similarly,
\begin{align*}
    Q_t=V_1^\top G_{2,t}(X_t)=Bg_2(Y_t,t)+\rho_2(t)V_1^\top U_t.
\end{align*}
From the expression of $Z_t$, we have
\begin{align*}
    B^\dagger Z_t=B^\dagger BY_t+B^\dagger V_1^\top U_t.
\end{align*}
Therefore,
\begin{align*}
    Y_t-B^\dagger Z_t=(I_{d_2}-B^\dagger B)Y_t-B^\dagger V_1^\top U_t.
\end{align*}
Moreover,
\begin{align*}
    Q_t-f_{\rm comp}(Z_t,t)=B\left[g_2(Y_t,t)-g_2(B^\dagger Z_t,t)\right]+\rho_2(t)V_1^\top U_t.
\end{align*}
Using the \(L_g(t)\)-Lipschitz continuity of \(g_2(\cdot,t)\), we get
\begin{align*}
    \left\|B\left[g_2(Y_t,t)-g_2(B^\dagger Z_t,t)\right]\right\|_2\leqslant \|B\|_{\rm op}L_g(t)\left\|Y_t-B^\dagger Z_t\right\|_2.
\end{align*}
Hence
\begin{align*}
    \|Q_t-f_{\rm comp}(Z_t,t)\|_2\leqslant \|B\|_{\rm op}L_g(t)\left\|Y_t-B^\dagger Z_t\right\|_2+\rho_2(t)\|V_1^\top U_t\|_2.
\end{align*}
Using \((a+b)^2\leqslant 2a^2+2b^2\), we obtain
\begin{align}
    \label{eq:info_upper_step_general}
    \mathcal I_t\leqslant 2\|B\|_{\rm op}^2L_g^2(t)\E\left[\left\|Y_t-B^\dagger Z_t\right\|_2^2\right]+2\rho_2^2(t)\E\left[\left\|V_1^\top U_t\right\|_2^2\right].
\end{align}

We now compute the two expectations. Since \(Y_t\) and \(U_t\) are independent and \(\E[U_t]=0\), the cross term in \(\E\|Y_t-B^\dagger Z_t\|_2^2\) vanishes. Therefore,
\begin{align*}
    \E\left[\left\|Y_t-B^\dagger Z_t\right\|_2^2\right]=\E\left[\left\|(I_{d_2}-B^\dagger B)Y_t\right\|_2^2\right]+\E\left[\left\|B^\dagger V_1^\top U_t\right\|_2^2\right].
\end{align*}
Since \(B^\dagger B\) is the orthogonal projector onto the row space of \(B\), \(I_{d_2}-B^\dagger B\) is also an orthogonal projector. Thus
\begin{align*}
    \E\left[\left\|(I_{d_2}-B^\dagger B)Y_t\right\|_2^2\right]=\operatorname{Tr}\left((I_{d_2}-B^\dagger B)M_{Y,2}(t)\right).
\end{align*}
Also, since \(\E[U_tU_t^\top]=\tilde h_2(t)P_2^\perp\),
\begin{align*}
    \E\left[\left\|B^\dagger V_1^\top U_t\right\|_2^2\right]=\tilde h_2(t)\operatorname{Tr}\left(B^\dagger V_1^\top P_2^\perp V_1(B^\dagger)^\top\right).
\end{align*}
Similarly,
\begin{align*}
    \E\left[\left\|V_1^\top U_t\right\|_2^2\right]=\tilde h_2(t)\operatorname{Tr}\left(V_1^\top P_2^\perp V_1\right).
\end{align*}
Substituting these identities into \eqref{eq:info_upper_step_general} yields
\begin{equation}
    \label{eq:info_upper_final_general}
\begin{aligned}
    \mathcal I_t&\leqslant 2\|B\|_{\rm op}^2L_g^2(t)\operatorname{Tr}\left((I_{d_2}-B^\dagger B)M_{Y,2}(t)\right)+2\|B\|_{\rm op}^2L_g^2(t)\tilde h_2(t)\operatorname{Tr}\left(B^\dagger V_1^\top P_2^\perp V_1(B^\dagger)^\top\right)\\
    &\quad+2\rho_2^2(t)\tilde h_2(t)\operatorname{Tr}\left(V_1^\top P_2^\perp V_1\right).
\end{aligned}
\end{equation}

Combining \eqref{eq:orthogonal_upper_bound_general}, and \eqref{eq:info_upper_final_general}, we obtain the general time-\(t\) bound. Integrating over \(t\in[t_0,T]\) gives the integrated bound.

It remains to derive the two principal-angle forms.

First suppose \(d_1\geqslant d_2\) and \(B\) has full column rank. Then \(B^\dagger B=I_{d_2}\), so
\begin{align*}
    \operatorname{Tr}\left((I_{d_2}-B^\dagger B)M_{Y,2}(t)\right)=0.
\end{align*}
The singular values of \(B=V_1^\top A_2\) are
\begin{align*}
    \sigma_j(B)=\cos\theta_j,\qquad j=1,\ldots,d_2.
\end{align*}
Hence
\begin{align*}
    \|P_{V_1}^\perp A_2\|_F^2=\sum_{j=1}^{d_2}\sin^2\theta_j.
\end{align*}
Moreover,
\begin{align*}
    \operatorname{Tr}(P_{V_1}^\perp P_2^\perp)=D-d_1-d_2+\sum_{j=1}^{d_2}\cos^2\theta_j,
\end{align*}
and
\begin{align*}
    \operatorname{Tr}(V_1^\top P_2^\perp V_1)=d_1-\sum_{j=1}^{d_2}\cos^2\theta_j.
\end{align*}
A direct calculation in principal-vector coordinates gives
\begin{align*}
    \operatorname{Tr}\left(B^\dagger V_1^\top P_2^\perp V_1(B^\dagger)^\top\right)=\sum_{j=1}^{d_2}\tan^2\theta_j.
\end{align*}
Using \(\rho_2^2(t)\tilde h_2(t)=\alpha^4(t)\sigma_2^4/\tilde h_2(t)\) gives the stated bound for \(d_1\geqslant d_2\).

Now suppose \(d_1<d_2\) and \(B\) has full row rank. Then \(B^\dagger B\) is the orthogonal projector onto a \(d_1\)-dimensional subspace of \(\mathbb R^{d_2}\), and the term
\begin{align*}
    \operatorname{Tr}\left((I_{d_2}-B^\dagger B)M_{Y,2}(t)\right)
\end{align*}
remains as an irreducible information-loss term. In this case \(r=d_1\), and
\begin{align*}
    \|P_{V_1}^\perp A_2\|_F^2=d_2-\sum_{j=1}^{d_1}\cos^2\theta_j,
\end{align*}
because \(A_2\) has \(d_2-d_1\) target directions that cannot be contained in \(\operatorname{col}(V_1)\). Furthermore,
\begin{align*}
    \operatorname{Tr}(P_{V_1}^\perp P_2^\perp)=D-d_1-d_2+\sum_{j=1}^{d_1}\cos^2\theta_j,
\end{align*}
and
\begin{align*}
    \operatorname{Tr}(V_1^\top P_2^\perp V_1)=d_1-\sum_{j=1}^{d_1}\cos^2\theta_j.
\end{align*}
The same principal-coordinate calculation gives
\begin{align*}
    \operatorname{Tr}\left(B^\dagger V_1^\top P_2^\perp V_1(B^\dagger)^\top\right)=\sum_{j=1}^{d_1}\tan^2\theta_j.
\end{align*}
Using again \(\rho_2^2(t)\tilde h_2(t)=\alpha^4(t)\sigma_2^4/\tilde h_2(t)\) yields the stated bound for \(d_1<d_2\). This proves the proposition.
\end{proof}

\begin{proposition}[ReLU approximation of the frozen-coordinate comparator]
\label{prop:relu_approx_frozen_comparator}
Suppose Assumptions~\ref{asm:noisy_low_dim_shift},~\ref{asm:p2_decay}, and~\ref{asm:s_paral_Lip} hold for the target distribution \(p_2\). Assume additionally that for every \(R>0\),
\begin{align*}
\tau_{\rm comp}(R):=\sup_{t\in[t_0,T],\,\|z\|_\infty\leqslant R}\|\partial_t f_{\rm comp}(z,t)\|_2<\infty.
\end{align*}
Let \(V_1\in\mathbb R^{D\times d_1}\) be the frozen projector with \(V_1^\top V_1=I_{d_1}\), and let
\begin{align*}
B:=V_1^\top A_2.
\end{align*}
Define
\begin{align*}
L_{\rm comp}:=\|B\|_{\rm op}\|B^\dagger\|_{\rm op}(\beta_2+1).
\end{align*}
Let
\begin{align*}
G_{\rm comp,0}:=\sup_{t\in[t_0,T]}\|f_{\rm comp}(0,t)\|_2.
\end{align*}
Recall the frozen-coordinate comparator
\begin{align*}
f_{\rm comp}(z,t):=Bg_2(B^\dagger z,t),\qquad z\in\mathbb R^{d_1},\quad t\in[t_0,T].
\end{align*}
Let
\begin{align*}
Z_t:=V_1^\top X_t,\qquad P_{Z,t}:=(V_1^\top)_\# p_{2,t}.
\end{align*}
For any approximation level $\epsilon>0$, choose
\begin{align*}
R_Z=C_{Z,1}\left(\sqrt{d_1}+\sqrt{d_2}+\sqrt{\log\frac{C_{Z,2}(1+G_{\rm comp,0}^2+L_{\rm comp}^2)}{\epsilon^2}}\right).
\end{align*}
Define
\begin{align*}
K_{\rm comp}(R_Z):=G_{\rm comp,0}+L_{\rm comp}\sqrt{d_1}R_Z.
\end{align*}
Then there exists a ReLU network
\begin{align*}
f_\theta\in\mathcal F_{\rm ReLU}^{(d_1)}(L,M,J,K,\kappa,\gamma,\gamma_t)
\end{align*}
such that
\begin{align*}
\sup_{t\in[t_0,T]}\left\|f_\theta(\cdot,t)-f_{\rm comp}(\cdot,t)\right\|_{L^2(P_{Z,t})}\leqslant \sqrt{d_1+1}\epsilon.
\end{align*}
The architecture parameters are chosen as in Lemma~\ref{lem:relu_box_approx} with
$m=d_1, F=f_{\rm comp}, R=R_Z, L_z=L_{\rm comp}, L_t=\tau_{\rm comp}(R_Z), K_0=K_{\rm comp}(R_Z).$
Consequently,
\begin{align*}
\inf_{f_\theta\in\mathcal F_{\rm ReLU}^{(d_1)}}\mathcal A_{\rm comp}(f_\theta;V_1)\leqslant (d_1+1)\epsilon^2\cdot \frac{1}{T-t_0}\int_{t_0}^T\frac{1}{h^2(t)}\,\rmd t.
\end{align*}
In particular, since \(h(t)=1-e^{-t}\), we have
\begin{align*}
\inf_{f_\theta\in\mathcal F_{\rm ReLU}^{(d_1)}}\mathcal A_{\rm comp}(f_\theta;V_1)\lesssim \frac{(d_1+1)\epsilon^2}{t_0(T-t_0)}\,,
\end{align*}
where \(\mathcal F_{\rm ReLU}^{(d_1)}\) denotes the above architecture class with the parameters chosen in Lemma~\ref{lem:relu_box_approx}. If \(T-t_0\) is fixed, this reduces to the order \(\epsilon^2/t_0\).
\end{proposition}

\begin{proof}[Proof of Proposition~\ref{prop:relu_approx_frozen_comparator}]
For each \(t\in[t_0,T]\), define the frozen-coordinate distribution
\begin{align*}
P_{Z,t}:=(V_1^\top)_\# p_{2,t}.
\end{align*}
Equivalently, if \(X_t\sim p_{2,t}\), then
\begin{align*}
Z_t:=V_1^\top X_t\sim P_{Z,t}.
\end{align*}
By the definition of \(\mathcal A_{\rm comp}\), for any \(f_\theta\in\mathcal F_{\rm ReLU}^{(d_1)}\),
\begin{align*}
\mathcal A_{\rm comp}(f_\theta;V_1)=\frac{1}{T-t_0}\int_{t_0}^T\frac{1}{h^2(t)}\left\|f_\theta(\cdot,t)-f_{\rm comp}(\cdot,t)\right\|_{L^2(P_{Z,t})}^2\,\rmd t.
\end{align*}
Therefore, it suffices to construct \(f_\theta\in\mathcal F_{\rm ReLU}^{(d_1)}(L,M,J,K,\kappa,\gamma,\gamma_t)\) such that
\begin{align*}
\sup_{t\in[t_0,T]}\left\|f_\theta(\cdot,t)-f_{\rm comp}(\cdot,t)\right\|_{L^2(P_{Z,t})}\leqslant \sqrt{d_1+1}\epsilon.
\end{align*}
Indeed, such a bound immediately implies
\begin{align*}
\mathcal A_{\rm comp}(f_\theta;V_1)\leqslant (d_1+1)\epsilon^2\cdot\frac{1}{T-t_0}\int_{t_0}^T\frac{1}{h^2(t)}\,\rmd t.
\end{align*}
Since \(h(t)=1-e^{-t}\), there exists an absolute constant \(C>0\) such that
\begin{align*}
\frac{1}{T-t_0}\int_{t_0}^T\frac{1}{h^2(t)}\,\rmd t\leqslant \frac{C}{t_0(T-t_0)}.
\end{align*}
Hence
\begin{align*}
\mathcal A_{\rm comp}(f_\theta;V_1)\lesssim \dfrac{C(d_1+1)\epsilon^2}{t_0(T-t_0)}.
\end{align*}
Thus, the remaining task is to prove the uniform-in-time \(L^2(P_{Z,t})\) approximation bound for \(f_{\rm comp}\).

\textbf{Step 1: Regularity transfer from \(f_2^*\) to \(f_{\rm comp}\).}
Recall that
\begin{align*}
f_2^*(y,t)=\tilde h_2(t)\nabla\log p_{2,t}^{\rm LD}(y)+y,\qquad g_2(y,t)=h(t)\nabla\log p_{2,t}^{\rm LD}(y)+y.
\end{align*}
Therefore,
\begin{align*}
g_2(y,t)=\frac{h(t)}{\tilde h_2(t)}f_2^*(y,t)+\left(1-\frac{h(t)}{\tilde h_2(t)}\right)y.
\end{align*}
By Assumption~\ref{asm:s_paral_Lip}, \(f_2^*(\cdot,t)\) is \(\beta_2\)-Lipschitz uniformly over \(t\in[0,T]\). Since \(0\leqslant h(t)/\tilde h_2(t)\leqslant 1\), for any \(y_1,y_2\in\mathbb R^{d_2}\) and \(t\in[t_0,T]\),
\begin{align*}
\|g_2(y_1,t)-g_2(y_2,t)\|_2\leqslant(\beta_2+1)\|y_1-y_2\|_2.
\end{align*}
Now recall
\begin{align*}
f_{\rm comp}(z,t)=Bg_2(B^\dagger z,t),\qquad B=V_1^\top A_2.
\end{align*}
Define
\begin{align*}
L_{\rm comp}:=\|B\|_{\rm op}\|B^\dagger\|_{\rm op}(\beta_2+1).
\end{align*}
Then, for any \(z_1,z_2\in\mathbb R^{d_1}\) and \(t\in[t_0,T]\),
\begin{align*}
\|f_{\rm comp}(z_1,t)-f_{\rm comp}(z_2,t)\|_2\leqslant L_{\rm comp}\|z_1-z_2\|_2.
\end{align*}
Indeed,
\begin{align*}
\|f_{\rm comp}(z_1,t)-f_{\rm comp}(z_2,t)\|_2=\|B[g_2(B^\dagger z_1,t)-g_2(B^\dagger z_2,t)]\|_2.
\end{align*}
Using the Lipschitz continuity of \(g_2(\cdot,t)\), we obtain
\begin{align*}
\|B[g_2(B^\dagger z_1,t)-g_2(B^\dagger z_2,t)]\|_2\leqslant \|B\|_{\rm op}(\beta_2+1)\|B^\dagger(z_1-z_2)\|_2.
\end{align*}
Thus,
\begin{align*}
\|f_{\rm comp}(z_1,t)-f_{\rm comp}(z_2,t)\|_2\leqslant \|B\|_{\rm op}\|B^\dagger\|_{\rm op}(\beta_2+1)\|z_1-z_2\|_2=L_{\rm comp}\|z_1-z_2\|_2.
\end{align*}
\textbf{Step 2: Growth and time regularity of \(f_{\rm comp}\).}
We next record two regularity properties that will be used in the truncation argument and in the ReLU construction on the compact domain.

First, \(f_{\rm comp}\) has at most linear growth in \(z\). Let
\begin{align*}
G_{2,0}:=\sup_{t\in[t_0,T]}\|g_2(0,t)\|_2.
\end{align*}
Under the assumed regularity of the target-domain core map on the compact interval $[t_0,T]$, $G_{2,0}<\infty$. Define
\begin{align*}
G_{\rm comp,0}:=\sup_{t\in[t_0,T]}\|f_{\rm comp}(0,t)\|_2\leqslant \|B\|_{\rm op}G_{2,0}.
\end{align*}
Then, for any \(z\in\mathbb R^{d_1}\) and \(t\in[t_0,T]\),
\begin{align*}
\|f_{\rm comp}(z,t)\|_2\leqslant G_{\rm comp,0}+L_{\rm comp}\|z\|_2.
\end{align*}
Indeed, by the definition of \(f_{\rm comp}\),
\begin{align*}
\|f_{\rm comp}(z,t)\|_2\leqslant \|B\|_{\rm op}\|g_2(B^\dagger z,t)\|_2.
\end{align*}
Using the Lipschitz bound for \(g_2\), we have
\begin{align*}
\|g_2(B^\dagger z,t)\|_2\leqslant \|g_2(0,t)\|_2+(\beta_2+1)\|B^\dagger z\|_2.
\end{align*}
Therefore,
\begin{align*}
\|f_{\rm comp}(z,t)\|_2\leqslant \|B\|_{\rm op}G_{2,0}+\|B\|_{\rm op}\|B^\dagger\|_{\rm op}(\beta_2+1)\|z\|_2=G_{\rm comp,0}+L_{\rm comp}\|z\|_2.
\end{align*}

Second, define the time-regularity modulus of the comparator on the truncated domain by
\begin{align*}
\tau_{\rm comp}(R):=\sup_{t\in[t_0,T],\,\|z\|_\infty\leqslant R}\|\partial_t f_{\rm comp}(z,t)\|_2.
\end{align*}
This quantity will play the role of the time-Lipschitz constant in the construction of the ReLU approximation. Since \(B\) and \(B^\dagger\) are fixed in time, whenever \(g_2\) is differentiable in \(t\), we have
\begin{align*}
\partial_t f_{\rm comp}(z,t)=B\partial_t g_2(B^\dagger z,t).
\end{align*}
In particular, if
\begin{align*}
\tau_{g,2}(R):=\sup_{t\in[t_0,T],\,\|y\|_2\leqslant R}\|\partial_t g_2(y,t)\|_2
\end{align*}
is finite, then
\begin{align*}
\tau_{\rm comp}(R)\leqslant \|B\|_{\rm op}\tau_{g,2}(\|B^\dagger\|_{\rm op}\sqrt{d_1}R).
\end{align*}

\textbf{Step 3: Tail control of the frozen-coordinate distribution.}
We now show that the frozen-coordinate distribution \(P_{Z,t}\) has a uniformly controlled tail over \(t\in[t_0,T]\). Under Assumption~\ref{asm:noisy_low_dim_shift}, we may write
\begin{align*}
X_t=\alpha(t)A_2z+\alpha(t)\epsilon_2+\sqrt{h(t)}\xi,
\end{align*}
where \(z\sim p_{2,z}\), \(\epsilon_2\sim\mathcal N(0,\sigma_2^2I_D)\), \(\xi\sim\mathcal N(0,I_D)\), and these random variables are independent. Hence
\begin{align*}
Z_t=V_1^\top X_t=\alpha(t)Bz+\alpha(t)V_1^\top\epsilon_2+\sqrt{h(t)}V_1^\top\xi.
\end{align*}
Since \(\|B\|_{\rm op}\leqslant1\), \(\alpha(t)\leqslant1\), and \(h(t)\leqslant1\), we have
\begin{align*}
\|Z_t\|_2\leqslant \|z\|_2+\|G_t\|_2,
\end{align*}
where
\begin{align*}
G_t:=\alpha(t)V_1^\top\epsilon_2+\sqrt{h(t)}V_1^\top\xi\sim\mathcal N(0,(\alpha^2(t)\sigma_2^2+h(t))I_{d_1}).
\end{align*}
In particular,
\begin{align*}
\alpha^2(t)\sigma_2^2+h(t)\leqslant \sigma_2^2+1.
\end{align*}
By Assumption~\ref{asm:p2_decay}, \(z\) has a sub-Gaussian-type norm tail. Therefore, there exist dimension-free constants \(c_Z,C_Z,C_{Z,0}>0\), depending only on the tail constants of \(p_{2,z}\) and on \(\sigma_2\), such that for all \(u\geqslant0\),
\begin{align*}
\sup_{t\in[t_0,T]}\mathbb P\left(\|Z_t\|_2>C_{Z,0}(\sqrt{d_2}+\sqrt{d_1})+u\right)\leqslant C_Z\exp(-c_Zu^2).
\end{align*}
Consequently, there exists a constant \(C_Z'>0\) such that for all \(R\geqslant 2C_{Z,0}(\sqrt{d_2}+\sqrt{d_1})\),
\begin{align*}
\sup_{t\in[t_0,T]}\E\left[(1+\|Z_t\|_2^2)\mathbf 1_{\{\|Z_t\|_2>R\}}\right]\leqslant C_Z'(1+R^2)\exp(-c_ZR^2/C_Z')\,,
\end{align*}
which follows from 
\begin{align*}
\E[\|Z_t\|_2^2\mathbf 1_{\{\|Z_t\|_2>R\}}]=R^2\mathbb P(\|Z_t\|_2>R)+2\int_R^\infty u\mathbb P(\|Z_t\|_2>u)\,\rmd u
\end{align*}
together with the preceding tail bound.
In particular, choosing
\begin{align*}
R_Z=C_{Z,1}\left(\sqrt{d_1}+\sqrt{d_2}+\sqrt{\log\frac{C_{Z,2}(1+G_{\rm comp,0}^2+L_{\rm comp}^2)}{\epsilon^2}}\right)
\end{align*}
with sufficiently large constants \(C_{Z,1},C_{Z,2}>0\), and enlarging \(C_{Z,1}\) to absorb the polynomial factor generated by \(K_{\rm comp}(R_Z)\), gives
\begin{align*}
\sup_{t\in[t_0,T]}\E\left[(1+\|Z_t\|_2^2)\mathbf 1_{\{\|Z_t\|_2>R_Z\}}\right]\leqslant \frac{\epsilon^2}{C\left(1+K_{\rm comp}^2(R_Z)+G_{\rm comp,0}^2+L_{\rm comp}^2\right)}.
\end{align*}
This is possible because \(K_{\rm comp}^2(R_Z)\leqslant C(1+G_{\rm comp,0}^2+L_{\rm comp}^2d_1R_Z^2)\), while the tail bound decays exponentially in \(R_Z^2\).

\textbf{Step 4: ReLU approximation on the truncated domain.}
Let
\begin{align*}
\mathcal D_Z(R_Z):=[-R_Z,R_Z]^{d_1}\times[t_0,T].
\end{align*}
By Steps 1 and 2, \(f_{\rm comp}\) satisfies the assumptions of Lemma~\ref{lem:relu_box_approx} on \(\mathcal D_Z(R_Z)\) with
\begin{align*}
m=q=d_1,\qquad L_z=L_{\rm comp},\qquad L_t=\tau_{\rm comp}(R_Z),\qquad K_0=K_{\rm comp}(R_Z).
\end{align*}
Applying Lemma~\ref{lem:relu_box_approx}, we obtain a network \(f_\theta\in\mathcal F_{\rm ReLU}^{(d_1)}(L,M,J,K,\kappa,\gamma,\gamma_t)\) such that
\begin{align*}
\sup_{(z,t)\in\mathcal D_Z(R_Z)}\|f_\theta(z,t)-f_{\rm comp}(z,t)\|_\infty\leqslant\epsilon.
\end{align*}
Therefore,
\begin{align*}
\sup_{(z,t)\in\mathcal D_Z(R_Z)}\|f_\theta(z,t)-f_{\rm comp}(z,t)\|_2\leqslant\sqrt{d_1}\epsilon.
\end{align*}
Moreover,
\begin{align*}
\sup_{z\in\mathbb R^{d_1},t\in[t_0,T]}\|f_\theta(z,t)\|_2\leqslant CK_{\rm comp}(R_Z).
\end{align*}

\textbf{Step 5: Compact-tail decomposition of the \(L^2(P_{Z,t})\) error.}
For each \(t\in[t_0,T]\), define the truncation event
\begin{align*}
\mathcal E_t:=\{\|Z_t\|_\infty\leqslant R_Z\}.
\end{align*}
Then
\begin{align*}
&\quad\E\left[\|f_\theta(Z_t,t)-f_{\rm comp}(Z_t,t)\|_2^2\right]\\
&=\E\left[\|f_\theta(Z_t,t)-f_{\rm comp}(Z_t,t)\|_2^2\mathbf 1_{\mathcal E_t}\right]+\E\left[\|f_\theta(Z_t,t)-f_{\rm comp}(Z_t,t)\|_2^2\mathbf 1_{\mathcal E_t^c}\right].
\end{align*}
On \(\mathcal E_t\), we have \((Z_t,t)\in[-R_Z,R_Z]^{d_1}\times[t_0,T]\). Therefore, by the compact-domain approximation constructed in Step 4,
\begin{align*}
\E\left[\|f_\theta(Z_t,t)-f_{\rm comp}(Z_t,t)\|_2^2\mathbf 1_{\mathcal E_t}\right]\leqslant d_1\epsilon^2.
\end{align*}
It remains to control the tail contribution. By the global boundedness of the constructed network and the linear growth bound of \(f_{\rm comp}\), we have
\begin{align*}
\|f_\theta(Z_t,t)-f_{\rm comp}(Z_t,t)\|_2^2\leqslant 2C^2K_{\rm comp}^2(R_Z)+2(G_{\rm comp,0}+L_{\rm comp}\|Z_t\|_2)^2.
\end{align*}
Hence, for another absolute constant \(C>0\),
\begin{align*}
\|f_\theta(Z_t,t)-f_{\rm comp}(Z_t,t)\|_2^2\leqslant C\left(1+K_{\rm comp}^2(R_Z)+G_{\rm comp,0}^2+L_{\rm comp}^2\right)(1+\|Z_t\|_2^2).
\end{align*}
Since \(\mathcal E_t^c=\{\|Z_t\|_\infty>R_Z\}\subseteq\{\|Z_t\|_2>R_Z\}\), the tail estimate from Step 3 gives
\begin{align*}
&\quad\sup_{t\in[t_0,T]}\E\left[\|f_\theta(Z_t,t)-f_{\rm comp}(Z_t,t)\|_2^2\mathbf 1_{\mathcal E_t^c}\right]\\
&\leqslant C\left(1+K_{\rm comp}^2(R_Z)+G_{\rm comp,0}^2+L_{\rm comp}^2\right)\sup_{t\in[t_0,T]}\E\left[(1+\|Z_t\|_2^2)\mathbf 1_{\{\|Z_t\|_2>R_Z\}}\right].
\end{align*}
By the choice of \(R_Z\) in Step 3, we have
\begin{align*}
\sup_{t\in[t_0,T]}\E\left[(1+\|Z_t\|_2^2)\mathbf 1_{\{\|Z_t\|_2>R_Z\}}\right]\leqslant \frac{\epsilon^2}{C\left(1+K_{\rm comp}^2(R_Z)+G_{\rm comp,0}^2+L_{\rm comp}^2\right)},
\end{align*}
Therefore,
\begin{align*}
\sup_{t\in[t_0,T]}\E\left[\|f_\theta(Z_t,t)-f_{\rm comp}(Z_t,t)\|_2^2\mathbf 1_{\mathcal E_t^c}\right]\leqslant \epsilon^2.
\end{align*}
Such a choice of \(R_Z\) is possible because the frozen-coordinate tail is sub-Gaussian by Step 3, while \(K_{\rm comp}(R_Z)\) grows at most linearly in \(R_Z\). Thus the polynomial factor in \(R_Z\) can be absorbed into the exponential tail by increasing the logarithmic constant in the choice of \(R_Z\).
Combining the compact and tail parts yields
\begin{align*}
\sup_{t\in[t_0,T]}\E\left[\|f_\theta(Z_t,t)-f_{\rm comp}(Z_t,t)\|_2^2\right]\leqslant (d_1+1)\epsilon^2.
\end{align*}
Equivalently,
\begin{align*}
\sup_{t\in[t_0,T]}\left\|f_\theta(\cdot,t)-f_{\rm comp}(\cdot,t)\right\|_{L^2(P_{Z,t})}\leqslant \sqrt{d_1+1}\epsilon.
\end{align*}

\textbf{Step 6: From \(L^2(P_{Z,t})\) approximation to \(\mathcal A_{\rm comp}\).}
By Step 5, the constructed ReLU network \(f_\theta\in\mathcal F_{\rm ReLU}^{(d_1)}\) satisfies
\begin{align*}
\sup_{t\in[t_0,T]}\left\|f_\theta(\cdot,t)-f_{\rm comp}(\cdot,t)\right\|_{L^2(P_{Z,t})}\leqslant\sqrt{d_1+1}\epsilon.
\end{align*}
Substituting this bound into the definition of \(\mathcal A_{\rm comp}\), we obtain
\begin{align*}
\mathcal A_{\rm comp}(f_\theta;V_1)\leqslant(d_1+1)\epsilon^2\cdot\frac{1}{T-t_0}\int_{t_0}^T\frac{1}{h^2(t)}\,\rmd t.
\end{align*}
Since \(h(t)=1-e^{-t}\), for \(t\in[t_0,T]\) we have
\begin{align*}
h(t)\geqslant \frac{1-e^{-T}}{T}t.
\end{align*}
Therefore,
\begin{align*}
\int_{t_0}^T\frac{1}{h^2(t)}\,\rmd t\leqslant \left(\frac{T}{1-e^{-T}}\right)^2\int_{t_0}^T\frac{1}{t^2}\,\rmd t\leqslant \left(\frac{T}{1-e^{-T}}\right)^2\frac{1}{t_0}.
\end{align*}
It follows that
\begin{align*}
\mathcal A_{\rm comp}(f_\theta;V_1)\leqslant(d_1+1)\epsilon^2\cdot\frac{1}{T-t_0}\left(\frac{T}{1-e^{-T}}\right)^2\frac{1}{t_0}.
\end{align*}
Consequently,
\begin{align*}
\inf_{f_\theta\in\mathcal F_{\rm ReLU}^{(d_1)}}\mathcal A_{\rm comp}(f_\theta;V_1)\leqslant C_T(d_1+1)\frac{\epsilon^2}{t_0(T-t_0)},
\end{align*}
where
\begin{align*}
C_T:=\left(\frac{T}{1-e^{-T}}\right)^2.
\end{align*}
If \(T-t_0\) is treated as a fixed constant, this is of order
\begin{align*}
\frac{(d_1+1)\epsilon^2}{t_0}.
\end{align*}
This proves the proposition.
\end{proof}

\begin{proof}[Proof of Theorem~\ref{thm:frozen_projector_statistical_bound}]

\textbf{Step 1: Target noisy truncation region and sample truncation event.}
For \(X_0\sim p_2\), write
\begin{align*}
Y:=A_2^\top X_0=z+A_2^\top\epsilon_2,\qquad U:=P_2^\perp X_0=P_2^\perp\epsilon_2.
\end{align*}
Define
\begin{align*}
\mathcal T_2(R_z,R_\perp):=\left\{x\in\mathbb R^D:\|A_2^\top x\|_2\leqslant R_z,\ \|P_2^\perp x\|_2\leqslant R_\perp\right\}.
\end{align*}
Let
\begin{align*}
R_*:=\sqrt{R_z^2+R_\perp^2}.
\end{align*}
On \(\mathcal T_2(R_z,R_\perp)\), we have
\begin{align*}
\|x\|_2^2=\|A_2^\top x\|_2^2+\|P_2^\perp x\|_2^2\leqslant R_*^2.
\end{align*}
By Assumption~\ref{asm:p2_decay} and standard Gaussian norm concentration, the choices
\begin{align*}
R_z=C_z(1+\sigma_2)\left(\sqrt{d_2}+\sqrt{\log\frac{8N_2}{\delta}}\right),
\end{align*}
and
\begin{align*}
R_\perp=C_\perp\sigma_2\left(\sqrt{D-d_2}+\sqrt{\log\frac{8N_2}{\delta}}\right)
\end{align*}
ensure that
\begin{align*}
\mathbb P\left(\|Y\|_2>R_z\right)\leqslant \frac{\delta}{8N_2},\qquad \mathbb P\left(\|U\|_2>R_\perp\right)\leqslant \frac{\delta}{8N_2}.
\end{align*}
Therefore,
\begin{align*}
\mathbb P_{X_0\sim p_2}\left(X_0\notin\mathcal T_2(R_z,R_\perp)\right)\leqslant \frac{\delta}{4N_2}.
\end{align*}
Let
\begin{align*}
\Omega_{\mathcal T}:=\bigcap_{i=1}^{N_2}\{x_i\in\mathcal T_2(R_z,R_\perp)\}.
\end{align*}
By the union bound,
\begin{align*}
\mathbb P(\Omega_{\mathcal T}^c)\leqslant \frac{\delta}{4}.
\end{align*}
Define the truncated population and empirical denoising risks by
\begin{align*}
\mathcal R_2^{\rm trunc}(s):=\E_{X_0\sim p_2}\left[\ell_2(X_0;s)\mathbf 1_{\mathcal T_2(R_z,R_\perp)}(X_0)\right],
\end{align*}
and
\begin{align*}
\hat{\mathcal R}_2^{\rm trunc}(s):=\frac{1}{N_2}\sum_{i=1}^{N_2}\ell_2(x_i;s)\mathbf 1_{\mathcal T_2(R_z,R_\perp)}(x_i).
\end{align*}
On \(\Omega_{\mathcal T}\), for every \(s\in\mathcal S_{\rm freeze}(V_1)\), one has
\begin{align*}
\hat{\mathcal R}_2^{\rm trunc}(s)=\hat{\mathcal R}_2(s).
\end{align*}

\textbf{Step 2: Deterministic oracle decomposition on the truncation event.}
For \(a\in(0,1)\), define the statistical error term
\begin{align*}
\mathfrak S_2(a;R_z,R_\perp)&:=\sup_{s\in\mathcal S_{\rm freeze}(V_1)}\left[\mathcal R_2^{\rm trunc}(s)-(1+a)\hat{\mathcal R}_2^{\rm trunc}(s)\right]_+\\
&\quad+(1+a)\left[\hat{\mathcal R}_2^{\rm trunc}(\bar s_\epsilon)-(1+a)\mathcal R_2^{\rm trunc}(\bar s_\epsilon)\right]_+.
\end{align*}
Define the one-sided uncentered truncation term
\begin{align*}
\mathfrak T_2(R_z,R_\perp):=\sup_{s\in\mathcal S_{\rm freeze}(V_1)}\left[\mathcal R_2(s)-\mathcal R_2^{\rm trunc}(s)\right].
\end{align*}
Since \(\ell_2(x;s)\geqslant0\), this is equivalently
\begin{align*}
\mathfrak T_2(R_z,R_\perp)=\sup_{s\in\mathcal S_{\rm freeze}(V_1)}\E_{X_0\sim p_2}\left[\ell_2(X_0;s)\mathbf 1_{\mathcal T_2^c(R_z,R_\perp)}(X_0)\right].
\end{align*}
On the event \(\Omega_{\mathcal T}\), since \(\hat s_{\rm freeze}\) minimizes \(\hat{\mathcal R}_2(s)\) over \(\mathcal S_{\rm freeze}(V_1)\), we have
\begin{align*}
\hat{\mathcal R}_2^{\rm trunc}(\hat s_{\rm freeze})=\hat{\mathcal R}_2(\hat s_{\rm freeze})\leqslant \hat{\mathcal R}_2(\bar s_\epsilon)=\hat{\mathcal R}_2^{\rm trunc}(\bar s_\epsilon).
\end{align*}
Therefore,
\begin{align*}
\mathcal L_2(\hat s_{\rm freeze})&=\mathcal R_2(\hat s_{\rm freeze})-E_2\\
&\leqslant \mathfrak T_2(R_z,R_\perp)+\mathcal R_2^{\rm trunc}(\hat s_{\rm freeze})-E_2\\
&\leqslant \mathfrak T_2(R_z,R_\perp)+\left[\mathcal R_2^{\rm trunc}(\hat s_{\rm freeze})-(1+a)\hat{\mathcal R}_2^{\rm trunc}(\hat s_{\rm freeze})\right]\\
&\quad+(1+a)\hat{\mathcal R}_2^{\rm trunc}(\hat s_{\rm freeze})-E_2\\
&\leqslant \mathfrak T_2(R_z,R_\perp)+\sup_{s\in\mathcal S_{\rm freeze}(V_1)}\left[\mathcal R_2^{\rm trunc}(s)-(1+a)\hat{\mathcal R}_2^{\rm trunc}(s)\right]_+\\
&\quad+(1+a)\hat{\mathcal R}_2^{\rm trunc}(\bar s_\epsilon)-E_2.
\end{align*}
Next, decompose the fixed-comparator empirical term as
\begin{align*}
(1+a)\hat{\mathcal R}_2^{\rm trunc}(\bar s_\epsilon)&\leqslant (1+a)\left[\hat{\mathcal R}_2^{\rm trunc}(\bar s_\epsilon)-(1+a)\mathcal R_2^{\rm trunc}(\bar s_\epsilon)\right]_+\\
&\quad+(1+a)^2\mathcal R_2^{\rm trunc}(\bar s_\epsilon).
\end{align*}
Combining the previous two displays and using the definition of \(\mathfrak S_2(a;R_z,R_\perp)\), we obtain
\begin{align*}
\mathcal L_2(\hat s_{\rm freeze})\leqslant \mathfrak T_2(R_z,R_\perp)+\mathfrak S_2(a;R_z,R_\perp)+(1+a)^2\mathcal R_2^{\rm trunc}(\bar s_\epsilon)-E_2.
\end{align*}
Since \(\ell_2(x;\bar s_\epsilon)\geqslant0\), we have
\begin{align*}
\mathcal R_2^{\rm trunc}(\bar s_\epsilon)\leqslant \mathcal R_2(\bar s_\epsilon).
\end{align*}
Therefore,
\begin{align*}
\mathcal L_2(\hat s_{\rm freeze})\leqslant \mathfrak T_2(R_z,R_\perp)+\mathfrak S_2(a;R_z,R_\perp)+(1+a)^2\mathcal R_2(\bar s_\epsilon)-E_2.
\end{align*}
Using
\begin{align*}
\mathcal R_2(\bar s_\epsilon)=\mathcal L_2(\bar s_\epsilon)+E_2,
\end{align*}
we obtain
\begin{align*}
\mathcal L_2(\hat s_{\rm freeze})\leqslant \mathfrak T_2(R_z,R_\perp)+\mathfrak S_2(a;R_z,R_\perp)+(1+a)^2\mathcal L_2(\bar s_\epsilon)+(2a+a^2)E_2.
\end{align*}
Finally, by the frozen-comparator comparison and Proposition~\ref{prop:relu_approx_frozen_comparator}, for any \(\eta\in(0,1)\),
\begin{align*}
\mathcal L_2(\bar s_\epsilon)\leqslant (1+\eta)\mathcal U_{\rm comp}(V_1)+\left(1+\frac{1}{\eta}\right)C_T(d_1+1)\frac{\epsilon^2}{t_0(T-t_0)}.
\end{align*}
Substituting this bound gives
\begin{align*}
\mathcal L_2(\hat s_{\rm freeze})&\leqslant \mathfrak T_2(R_z,R_\perp)+\mathfrak S_2(a;R_z,R_\perp)\\
&\quad+(1+a)^2(1+\eta)\mathcal U_{\rm comp}(V_1)\\
&\quad+(1+a)^2\left(1+\frac{1}{\eta}\right)C_T(d_1+1)\frac{\epsilon^2}{t_0(T-t_0)}+(2a+a^2)E_2.
\end{align*}

\textbf{Step 3: Control of the truncation error \(\mathfrak T_2(R_z,R_\perp)\).}
Recall that
\begin{align*}
\mathfrak T_2(R_z,R_\perp)=\sup_{s\in\mathcal S_{\rm freeze}(V_1)}\E_{X_0\sim p_2}\left[\ell_2(X_0;s)\mathbf 1_{\mathcal T_2^c(R_z,R_\perp)}(X_0)\right].
\end{align*}
For \(s=s_{V_1,\theta}\in\mathcal S_{\rm freeze}(V_1)\), using
\begin{align*}
s_{V_1,\theta}(x,t)=\frac{1}{h(t)}V_1f_\theta(V_1^\top x,t)-\frac{1}{h(t)}x,
\end{align*}
and
\begin{align*}
\nabla_{X_t}\log\phi_t(X_t\mid X_0)=\frac{\alpha(t)X_0-X_t}{h(t)},
\end{align*}
we have, for \(X_t\mid X_0=x\),
\begin{align*}
\nabla_{X_t}\log\phi_t(X_t\mid x)-s_{V_1,\theta}(X_t,t)=\frac{\alpha(t)x-V_1f_\theta(V_1^\top X_t,t)}{h(t)}.
\end{align*}
Since \(\|V_1f_\theta(V_1^\top X_t,t)\|_2\leqslant K\) and \(\alpha(t)\leqslant1\), it follows that
\begin{align*}
\left\|\nabla_{X_t}\log\phi_t(X_t\mid x)-s_{V_1,\theta}(X_t,t)\right\|_2^2\leqslant \frac{2\|x\|_2^2+2K^2}{h^2(t)}.
\end{align*}
Therefore,
\begin{align*}
\ell_2(x;s_{V_1,\theta})\leqslant \frac{2(K^2+\|x\|_2^2)}{T-t_0}\int_{t_0}^T\frac{1}{h^2(t)}\,\rmd t.
\end{align*}
Using
\begin{align*}
\frac{1}{T-t_0}\int_{t_0}^T\frac{1}{h^2(t)}\,\rmd t\leqslant \frac{C_T}{t_0(T-t_0)},
\end{align*}
we obtain
\begin{align*}
\ell_2(x;s_{V_1,\theta})\leqslant \frac{C_T(K^2+\|x\|_2^2)}{t_0(T-t_0)}.
\end{align*}
Consequently,
\begin{align*}
\mathfrak T_2(R_z,R_\perp)\leqslant \frac{C_T}{t_0(T-t_0)}\E_{X_0\sim p_2}\left[(K^2+\|X_0\|_2^2)\mathbf 1_{\mathcal T_2^c(R_z,R_\perp)}(X_0)\right].
\end{align*}
Since \(X_0=A_2Y+U\) with \(Y=A_2^\top X_0\), \(U=P_2^\perp X_0\), and \(\|X_0\|_2^2=\|Y\|_2^2+\|U\|_2^2\), while
\begin{align*}
\mathcal T_2^c(R_z,R_\perp)\subseteq\{\|Y\|_2>R_z\}\cup\{\|U\|_2>R_\perp\},
\end{align*}
the sub-Gaussian tail of \(Y=z+A_2^\top\epsilon_2\) and the Gaussian tail of \(U=P_2^\perp\epsilon_2\) imply
\begin{align*}
\E\left[(K^2+\|X_0\|_2^2)\mathbf 1_{\mathcal T_2^c(R_z,R_\perp)}(X_0)\right]\leqslant C(K^2+R_z^2+R_\perp^2)\frac{\delta}{N_2}.
\end{align*}
This follows from the standard identity
\begin{align*}
\E[W^2\mathbf 1_{\{W>R\}}]=R^2\mathbb P(W>R)+2\int_R^\infty u\mathbb P(W>u)\,\rmd u
\end{align*}
applied to \(W=\|Y\|_2\) and \(W=\|U\|_2\). Therefore,
\begin{align*}
\mathfrak T_2(R_z,R_\perp)\leqslant C\frac{K^2+R_z^2+R_\perp^2}{t_0(T-t_0)}\frac{\delta}{N_2}.
\end{align*}

\textbf{Step 4: Control of the statistical error \(\mathfrak S_2(a;R_z,R_\perp)\).}
For \(s\in\mathcal S_{\rm freeze}(V_1)\), define
\begin{align*}
g_s(x):=\ell_2(x;s)\mathbf 1_{\mathcal T_2(R_z,R_\perp)}(x).
\end{align*}
By the bound in Step 3 and the fact that \(\|x\|_2\leqslant R_*:=\sqrt{R_z^2+R_\perp^2}\) on \(\mathcal T_2(R_z,R_\perp)\), we have
\begin{align*}
0\leqslant g_s(x)\leqslant B_u:=C_T\frac{K^2+R_*^2}{t_0(T-t_0)}.
\end{align*}
Let
\begin{align*}
\mathcal N_{2}(\iota):=\mathcal N\left(\iota,\left\{g_s:s\in\mathcal S_{\rm freeze}(V_1)\right\},\|\cdot\|_\infty\right)
\end{align*}
be the \(\iota\)-covering number of the truncated loss class under the sup norm.
Applying the two-sided Bernstein-type concentration inequality in Lemma 15 of~\cite{Chen2023_scorelowdim} to the bounded class \(\{g_s:s\in\mathcal S_{\rm freeze}(V_1)\}\), together with an \(\iota\)-net argument, yields that with probability at least \(1-\delta/4\),
\begin{align*}
\sup_{s\in\mathcal S_{\rm freeze}(V_1)}\left[\mathcal R_2^{\rm trunc}(s)-(1+a)\hat{\mathcal R}_2^{\rm trunc}(s)\right]_+\leqslant (2+a)\iota+C\left(1+\frac{3}{a}\right)\frac{B_u}{N_2}\left(\log\mathcal N_2(\iota)+\log\frac{8}{\delta}\right).
\end{align*}
Since \(\bar s_\epsilon\) is fixed independently of the target-domain samples, another application of Lemma 15 gives, with probability at least \(1-\delta/4\),
\begin{align*}
\left[\hat{\mathcal R}_2^{\rm trunc}(\bar s_\epsilon)-(1+a)\mathcal R_2^{\rm trunc}(\bar s_\epsilon)\right]_+\leqslant C\left(1+\frac{6}{a}\right)\frac{B_u}{N_2}\log\frac{8}{\delta}.
\end{align*}
Combining the two displays and recalling the definition of \(\mathfrak S_2(a;R_z,R_\perp)\), we obtain that with probability at least \(1-\delta/2\),
\begin{align*}
\mathfrak S_2(a;R_z,R_\perp)&\leqslant (2+a)\iota+C\left(1+\frac{3}{a}\right)\frac{B_u}{N_2}\left(\log\mathcal N_2(\iota)+\log\frac{8}{\delta}\right)\\
&\quad+C(1+a)\left(1+\frac{6}{a}\right)\frac{B_u}{N_2}\log\frac{8}{\delta}.
\end{align*}
Therefore,
\begin{align*}
\mathfrak S_2(a;R_z,R_\perp)\leqslant (2+a)\iota+C\cdot C_T\left(1+\frac{1}{a}\right)\frac{K^2+R_*^2}{N_2t_0(T-t_0)}\left(\log\mathcal N_2(\iota)+\log\frac{8}{\delta}\right).
\end{align*}

Substituting Step 3 and 4 back, we have
\begin{align*}
\mathcal L_2(\hat s_{\rm freeze})&\leqslant C\frac{K^2+R_z^2+R_\perp^2}{t_0(T-t_0)}\frac{\delta}{N_2}\\
&\quad+(2+a)\iota+C\cdot C_T\left(1+\dfrac{1}{a}\right)\frac{K^2+R_z^2+R_\perp^2}{N_2t_0(T-t_0)}\left(\log\mathcal N_2(\iota)+\log\frac{8}{\delta}\right)\\
&\quad+(1+a)^2(1+\eta)\mathcal U_{\rm comp}(V_1)\\
&\quad+(1+a)^2\left(1+\frac{1}{\eta}\right)C_T(d_1+1)\frac{\epsilon^2}{t_0(T-t_0)}+(2a+a^2)E_2.
\end{align*}
Moreover, by the denoising score-matching identity,
\[
E_2\leqslant \frac{D}{T-t_0}\int_{t_0}^T\frac{1}{h(t)}\,\rmd t
=
\frac{D}{T-t_0}\log\frac{e^T-1}{e^{t_0}-1}.
\]
Then choosing $a=\epsilon^2,\eta=1,\iota=\dfrac{1}{N_2t_0(T-t_0)}$ derives the required result.
\end{proof}

\subsection{Proof of Section~\ref{sec:mixed_training}}
\label{app:mixed_training_proofs}

\begin{proof}[Proof of Proposition~\ref{prop:optimal_mixed_projector}]
Using \(P_W^\perp=I_D-P_W\), we have
\begin{align*}
    \|P_W^\perp A_i\|_F^2=\operatorname{Tr}(A_i^\top P_W^\perp A_i)=d_i-\operatorname{Tr}(W^\top A_iA_i^\top W).
\end{align*}
Therefore
\begin{align*}
    \Gamma_k(W)&=\omega_1c_1d_1+\omega_2c_2d_2-\operatorname{Tr}\left(W^\top M_{\rm mix}W\right).
\end{align*}
Since
\begin{align*}
    \operatorname{Tr}(M_{\rm mix})=\omega_1c_1d_1+\omega_2c_2d_2,
\end{align*}
minimizing \(\Gamma_k(W)\) is equivalent to maximizing
\begin{align*}
    \operatorname{Tr}(W^\top M_{\rm mix}W)
\end{align*}
over all \(D\times k\) matrices with orthonormal columns. By the Rayleigh--Ritz variational principle, the maximum equals
\begin{align*}
    \sum_{j=1}^k\lambda_j(M_{\rm mix}),
\end{align*}
and it is achieved by taking \(W\) to be the top \(k\) eigenvectors of \(M_{\rm mix}\). The monotonicity in \(k\) follows immediately from the eigenvalue formula. Finally, the residual is zero if and only if \(P_W^\perp A_1=0\) and \(P_W^\perp A_2=0\), which is equivalent to both component subspaces being contained in \(\operatorname{span}(W)\).
\end{proof}

\begin{corollary}[Principal-angle interpretation]
\label{cor:principal_angle_projector}
Let \(A_i\in\mathbb R^{D\times d_i}\) have orthonormal columns for \(i=1,2\). Set
\begin{align*}
a:=\omega_1c_1,\qquad b:=\omega_2c_2,\qquad M_{\rm mix}:=aA_1A_1^\top+bA_2A_2^\top.
\end{align*}
Let \(r:=\min\{d_1,d_2\}\), and let
\begin{align*}
0\leqslant\varphi_1\leqslant\cdots\leqslant\varphi_r\leqslant\pi/2
\end{align*}
be the principal angles between \(\operatorname{span}(A_1)\) and \(\operatorname{span}(A_2)\). Then \(M_{\rm mix}\) is zero on
\begin{align*}
\left(\operatorname{span}(A_1)+\operatorname{span}(A_2)\right)^\perp.
\end{align*}
On the two-dimensional principal-angle block associated with \(\varphi_j\), the eigenvalues of \(M_{\rm mix}\) are
\begin{align}
\label{eq:general_weighted_principal_eigs}
\lambda_j^\pm=\frac{a+b}{2}\pm\frac{1}{2}\sqrt{(a-b)^2+4ab\cos^2\varphi_j},\qquad j=1,\ldots,r.
\end{align}
If \(d_1>d_2\), then \(M_{\rm mix}\) has additionally \(d_1-d_2\) eigenvalues equal to \(a\). If \(d_2>d_1\), then \(M_{\rm mix}\) has additionally \(d_2-d_1\) eigenvalues equal to \(b\). All remaining eigenvalues are zero.
\end{corollary}

\begin{proof}[Proof of Corollary~\ref{cor:principal_angle_projector}]
By the cosine-sine decomposition for two subspaces, there exist orthonormal principal vectors \(u_j\in\operatorname{span}(A_1)\) and \(v_j\in\operatorname{span}(A_2)\) such that \(u_j^\top v_j=\cos\varphi_j\) for \(j=1,\ldots,r\). The action of \(P_{A_1}\) and \(P_{A_2}\) decomposes into independent principal-angle blocks, together with the extra directions belonging only to the larger subspace.

On the block spanned by \(u_j\) and the normalized component of \(v_j\) orthogonal to \(u_j\), the matrix \(aP_{A_1}+bP_{A_2}\) has trace \(a+b\) and determinant \(ab\sin^2\varphi_j\). Hence its two eigenvalues are
\begin{align*}
\lambda_j^\pm=\frac{a+b}{2}\pm\frac{1}{2}\sqrt{(a+b)^2-4ab\sin^2\varphi_j}=\frac{a+b}{2}\pm\frac{1}{2}\sqrt{(a-b)^2+4ab\cos^2\varphi_j}.
\end{align*}
If \(d_1>d_2\), the \(d_1-d_2\) directions in \(\operatorname{span}(A_1)\) orthogonal to \(\operatorname{span}(A_2)\) are eigenvectors with eigenvalue \(a\). If \(d_2>d_1\), the additional directions in \(\operatorname{span}(A_2)\) have eigenvalue \(b\). On the orthogonal complement of \(\operatorname{span}(A_1)+\operatorname{span}(A_2)\), both projectors vanish, so the eigenvalue is zero.
\end{proof}

In the special case \(d_1=d_2=d\), \(c_1=c_2=1\), and \(\omega_1=\pi,\omega_2=1-\pi\), we have \(a=\pi\) and \(b=1-\pi\), so \eqref{eq:general_weighted_principal_eigs} becomes
\begin{align}
\label{eq:weighted_principal_eigs}
\lambda_j^\pm=\frac{1}{2}\left[1\pm\sqrt{1-4\pi(1-\pi)\sin^2\varphi_j}\right],\qquad j=1,\ldots,d.
\end{align}
In the balanced case \(\pi=1/2\), this further reduces to
\begin{align*}
\lambda_j^\pm=\frac{1\pm\cos\varphi_j}{2}.
\end{align*}

Consequently, the optimal \(k\)-dimensional shared projector in Proposition~\ref{prop:optimal_mixed_projector} selects the directions corresponding to the largest eigenvalues of \(M_{\rm mix}\). Principal directions with small angles concentrate most of their mixed subspace energy in one direction, whereas directions with large angles require additional latent dimensions to be well represented. If \(\varphi_j=0\), then \(\lambda_j^-=0\), reflecting that the corresponding principal-angle block degenerates into a shared one-dimensional direction.

\begin{proof}[Proof of Theorem~\ref{thm:mixed_oracle_projector_comparison}]
\noindent\textbf{Score structure under mixed training.}
Let
\begin{align*}
p_{\rm mix}:=\omega_1p_1+\omega_2p_2,\qquad \omega_1+\omega_2=1,
\end{align*}
and denote by \(p_{{\rm mix},t}\) the corresponding forward-diffused density. Since the forward diffusion is linear,
\begin{align*}
p_{{\rm mix},t}(x)=\omega_1p_{1,t}(x)+\omega_2p_{2,t}(x).
\end{align*}
Define the posterior mixing weights
\begin{align*}
\pi_i(x,t):=\frac{\omega_ip_{i,t}(x)}{\omega_1p_{1,t}(x)+\omega_2p_{2,t}(x)},\qquad i=1,2.
\end{align*}
Then
\begin{align*}
\nabla\log p_{{\rm mix},t}(x)=\pi_1(x,t)\nabla\log p_{1,t}(x)+\pi_2(x,t)\nabla\log p_{2,t}(x).
\end{align*}
For each component, define
\begin{align*}
g_i(y,t):=h(t)\nabla\log p_{i,t}^{\rm LD}(y)+y,\qquad \rho_i(t):=1-\frac{h(t)}{\tilde h_i(t)}=\frac{\alpha^2(t)\sigma_i^2}{\tilde h_i(t)}.
\end{align*}
For \(i=1,2\), define
\begin{align*}
M_{g,i}(t):=\E_{X_t\sim p_{i,t}}\left[g_i(A_i^\top X_t,t)g_i(A_i^\top X_t,t)^\top\right],
\end{align*}
and
\begin{align*}
\lambda_{g,i}^{\max}(t):=\lambda_{\max}(M_{g,i}(t)).
\end{align*}
Let
\begin{align*}
P_i:=A_iA_i^\top,\qquad P_i^\perp:=I_D-P_i,
\end{align*}
and
\begin{align*}
G_{i,t}(x):=A_ig_i(A_i^\top x,t)+\rho_i(t)P_i^\perp x.
\end{align*}
Then the mixed score can be written as
\begin{align*}
\nabla\log p_{{\rm mix},t}(x)=\frac{1}{h(t)}G_{{\rm mix},t}(x)-\frac{1}{h(t)}x,
\end{align*}
where
\begin{align*}
G_{{\rm mix},t}(x):=\pi_1(x,t)G_{1,t}(x)+\pi_2(x,t)G_{2,t}(x).
\end{align*}

Let \(U\in\mathbb R^{D\times m}\) be an ideal latent projector with \(U^\top U=I_m\). In the comparison below, \(U\) will be instantiated either as the source-aligned projector \(V_1\) with \(m=d_1\), or as the shared mixed-data projector \(W_k\) with \(m=k\). The word ``ideal'' is used only at the oracle-analysis level: \(\mathcal B_{\rm mix}(U)\) measures the best approximation power of the score network class after a projector has been chosen, and does not describe the finite-sample cost of learning the projector itself.
Define
\begin{align*}
P_U:=UU^\top,\qquad P_U^\perp:=I_D-P_U,\qquad H_i(U):=U^\top A_i,\quad i=1,2.
\end{align*}
Consider the mixed-data score class induced by \(U\):
\begin{align*}
\mathcal S_{\rm mix}(U):=\left\{s_{U,\theta}(x,t)=\frac{1}{h(t)}Uf_\theta(U^\top x,t)-\frac{1}{h(t)}x:f_\theta\in\mathcal F_{\rm ReLU}^{(m)}\right\}.
\end{align*}
For any score network \(s\), define the mixed score-matching risk
\begin{align*}
\mathcal L_{\rm mix}(s):=\frac{1}{T-t_0}\int_{t_0}^T\E_{X_t\sim p_{{\rm mix},t}}\left[\|s(X_t,t)-\nabla\log p_{{\rm mix},t}(X_t)\|_2^2\right]\,\rmd t.
\end{align*}
The mixed ReLU oracle risk under the projector \(U\) is
\begin{align*}
\mathcal B_{\rm mix}(U):=\inf_{s\in\mathcal S_{\rm mix}(U)}\mathcal L_{\rm mix}(s).
\end{align*}

For \(i=1,2\), let \(q_{i,U,t}\) denote the density of \(U^\top X_t\) when \(X_t\sim p_{i,t}\). Define the projected posterior mixing weights by
\begin{align*}
\pi_i^U(z,t):=\frac{\omega_iq_{i,U,t}(z)}{\omega_1q_{1,U,t}(z)+\omega_2q_{2,U,t}(z)},\qquad i=1,2.
\end{align*}
Let \(C\in\{1,2\}\) denote the component label, with \(\mathbb P(C=i)=\omega_i\) and \(X_t\mid C=i\sim p_{i,t}\). Then
\begin{align*}
\pi_i^U(z,t)=\mathbb P(C=i\mid U^\top X_t=z).
\end{align*}
Thus, \(\pi_i^U\) is the Bayes posterior component weight after observing only the compressed coordinate \(U^\top X_t\). If the exact component conditional predictors
\begin{align*}
\mu_{i,U,t}(z):=\E\left[U^\top G_{i,t}(X_t)\mid U^\top X_t=z,\ C=i\right]
\end{align*}
were available, then the optimal measurable predictor of \(U^\top G_{{\rm mix},t}(X_t)\) given \(U^\top X_t=z\) would be
\begin{align*}
\sum_{i=1}^2\pi_i^U(z,t)\mu_{i,U,t}(z).
\end{align*}
As in the frozen-projector analysis, directly controlling these conditional predictors is difficult. We therefore use an explicit componentwise comparator:
\begin{align*}
\psi_{i,U}(z,t):=H_i(U)g_i(H_i(U)^\dagger z,t),\qquad i=1,2.
\end{align*}
The mixed explicit comparator is then defined as
\begin{align*}
f_{{\rm mix},U}^{\rm comp}(z,t):=\pi_1^U(z,t)\psi_{1,U}(z,t)+\pi_2^U(z,t)\psi_{2,U}(z,t).
\end{align*}
Finally, define
\begin{align*}
s_{{\rm mix},U}^{\rm comp}(x,t):=\frac{1}{h(t)}Uf_{{\rm mix},U}^{\rm comp}(U^\top x,t)-\frac{1}{h(t)}x,
\end{align*}
and
\begin{align*}
\mathcal U_{\rm mix}(U):=\mathcal L_{\rm mix}(s_{{\rm mix},U}^{\rm comp}).
\end{align*}

Define the ReLU approximation error of the mixed comparator under \(U\) by
\begin{align*}
\mathcal A_{\rm mix}(f_\theta;U):=\frac{1}{T-t_0}\int_{t_0}^T\frac{1}{h^2(t)}\E_{X_t\sim p_{{\rm mix},t}}\left[\left\|f_\theta(U^\top X_t,t)-f_{{\rm mix},U}^{\rm comp}(U^\top X_t,t)\right\|_2^2\right]\,\rmd t.
\end{align*}
Then, for any \(\eta>0\),
\begin{align*}
\mathcal B_{\rm mix}(U)\leqslant (1+\eta)\mathcal U_{\rm mix}(U)+\left(1+\frac{1}{\eta}\right)\inf_{f_\theta\in\mathcal F_{\rm ReLU}^{(m)}}\mathcal A_{\rm mix}(f_\theta;U).
\end{align*}
Indeed, for any \(f_\theta\in\mathcal F_{\rm ReLU}^{(m)}\), one decomposes
\begin{align*}
Uf_\theta(U^\top X_t,t)-G_{{\rm mix},t}(X_t)&=\left[Uf_{{\rm mix},U}^{\rm comp}(U^\top X_t,t)-G_{{\rm mix},t}(X_t)\right]\\
&\quad+U\left[f_\theta(U^\top X_t,t)-f_{{\rm mix},U}^{\rm comp}(U^\top X_t,t)\right],
\end{align*}
and applies \(\|a+b\|_2^2\leqslant(1+\eta)\|a\|_2^2+(1+1/\eta)\|b\|_2^2\), using \(U^\top U=I_m\).

The measurable-comparator risk decomposes into an output-side structural term and an in-space comparator term:
\begin{align*}
\mathcal U_{\rm mix}(U)=\mathcal O_{\rm mix}(U)+\mathcal I_{\rm mix}(U),
\end{align*}
where
\begin{align*}
\mathcal O_{\rm mix}(U):=\frac{1}{T-t_0}\int_{t_0}^T\frac{1}{h^2(t)}\E_{X_t\sim p_{{\rm mix},t}}\left[\left\|P_U^\perp G_{{\rm mix},t}(X_t)\right\|_2^2\right]\,\rmd t,
\end{align*}
and
\begin{align*}
\mathcal I_{\rm mix}(U):=\frac{1}{T-t_0}\int_{t_0}^T\frac{1}{h^2(t)}\E_{X_t\sim p_{{\rm mix},t}}\left[\left\|f_{{\rm mix},U}^{\rm comp}(U^\top X_t,t)-U^\top G_{{\rm mix},t}(X_t)\right\|_2^2\right]\,\rmd t.
\end{align*}
This follows by decomposing \(G_{{\rm mix},t}(X_t)\) into its \(P_U\)- and \(P_U^\perp\)-components and using the orthogonality of \(\operatorname{span}(U)\) and its complement.

The output-side term is controlled by the joint subspace residual. Since
\begin{align*}
G_{{\rm mix},t}(x)=\pi_1(x,t)G_{1,t}(x)+\pi_2(x,t)G_{2,t}(x),
\end{align*}
convexity gives
\begin{align*}
\left\|P_U^\perp G_{{\rm mix},t}(x)\right\|_2^2\leqslant\sum_{i=1}^2\pi_i(x,t)\left\|P_U^\perp G_{i,t}(x)\right\|_2^2.
\end{align*}
Using \(\pi_i(x,t)p_{{\rm mix},t}(x)=\omega_ip_{i,t}(x)\), we obtain
\begin{align*}
\E_{p_{{\rm mix},t}}\left[\left\|P_U^\perp G_{{\rm mix},t}(X_t)\right\|_2^2\right]\leqslant\sum_{i=1}^2\omega_i\E_{p_{i,t}}\left[\left\|P_U^\perp G_{i,t}(X_t)\right\|_2^2\right].
\end{align*}
For each component, the same calculation as in the frozen-projector analysis yields
\begin{align*}
\E_{p_{i,t}}\left[\left\|P_U^\perp G_{i,t}(X_t)\right\|_2^2\right]\leqslant\lambda_{g,i}^{\max}(t)\|P_U^\perp A_i\|_F^2+\frac{\alpha^4(t)\sigma_i^4}{\tilde h_i(t)}\operatorname{Tr}(P_U^\perp P_i^\perp).
\end{align*}
Therefore,
\begin{align*}
\mathcal O_{\rm mix}(U)\leqslant\sum_{i=1}^2\omega_i\bar c_i\|P_U^\perp A_i\|_F^2+\sum_{i=1}^2\omega_i\bar n_i\operatorname{Tr}(P_U^\perp P_i^\perp),
\end{align*}
where
\begin{align*}
\bar c_i:=\frac{1}{T-t_0}\int_{t_0}^T\frac{\lambda_{g,i}^{\max}(t)}{h^2(t)}\,\rmd t,\qquad \bar n_i:=\frac{1}{T-t_0}\int_{t_0}^T\frac{\alpha^4(t)\sigma_i^4}{h^2(t)\tilde h_i(t)}\,\rmd t.
\end{align*}
If \(c_i=\bar c_i\) in \(\Gamma_k(W)\), the first term is exactly the weighted joint residual
\begin{align*}
\sum_{i=1}^2\omega_i\bar c_i\|P_U^\perp A_i\|_F^2.
\end{align*}
In particular, for \(U=W_k\), where \(W_k\) is the top-\(k\) eigenspace of \(M_{\rm mix}=\omega_1\bar c_1A_1A_1^\top+\omega_2\bar c_2A_2A_2^\top\), we have
\begin{align*}
\mathcal O_{\rm mix}(W_k)\leqslant\sum_{j=k+1}^D\lambda_j(M_{\rm mix})+\sum_{i=1}^2\omega_i\bar n_i\operatorname{Tr}(P_{W_k}^\perp P_i^\perp).
\end{align*}
For \(U=V_1\), the same bound becomes
\begin{align*}
\mathcal O_{\rm mix}(V_1)\leqslant\omega_1\bar c_1\|P_{V_1}^\perp A_1\|_F^2+\omega_2\bar c_2\|P_{V_1}^\perp A_2\|_F^2+\sum_{i=1}^2\omega_i\bar n_i\operatorname{Tr}(P_{V_1}^\perp P_i^\perp).
\end{align*}
Thus, \(W_k\) improves the output-side structural term whenever the reduction in the joint residual offsets the additional dimension-dependent costs introduced later through approximation and statistical complexity.

It remains to interpret the in-space term \(\mathcal I_{\rm mix}(U)\). Define
\begin{align*}
Z:=U^\top X_t,\qquad e_{i,U}(X_t,t):=\psi_{i,U}(Z,t)-U^\top G_{i,t}(X_t).
\end{align*}
Then
\begin{align*}
&\quad f_{{\rm mix},U}^{\rm comp}(Z,t)-U^\top G_{{\rm mix},t}(X_t)\\
&=\sum_{i=1}^2\pi_i(X_t,t)e_{i,U}(X_t,t)+\sum_{i=1}^2(\pi_i^U(Z,t)-\pi_i(X_t,t))\psi_{i,U}(Z,t)\\
&=\sum_{i=1}^2\pi_i(X_t,t)e_{i,U}(X_t,t)+(\pi_1^U(Z,t)-\pi_1(X_t,t))(\psi_{1,U}(Z,t)-\psi_{2,U}(Z,t)).
\end{align*}
Consequently, by \(\|a+b\|_2^2\leqslant2\|a\|_2^2+2\|b\|_2^2\) and \(\|\sum_i\pi_i e_i\|_2^2\leqslant\sum_i\pi_i\|e_i\|_2^2\), we have
\begin{align*}
&\quad\mathcal I_{\rm mix}(U)\\
&\leqslant \frac{2}{T-t_0}\int_{t_0}^T\frac{1}{h^2(t)}\sum_{i=1}^2\omega_i\E_{X_t\sim p_{i,t}}\left[\left\|\psi_{i,U}(U^\top X_t,t)-U^\top G_{i,t}(X_t)\right\|_2^2\right]\,\rmd t\\
&\quad+\frac{2}{T-t_0}\int_{t_0}^T\frac{1}{h^2(t)}\E_{X_t\sim p_{{\rm mix},t}}\left[\left(\pi_1^U(U^\top X_t,t)-\pi_1(X_t,t)\right)^2\left\|\psi_{1,U}(U^\top X_t,t)-\psi_{2,U}(U^\top X_t,t)\right\|_2^2\right]\,\rmd t.
\end{align*}
The first term is a component reconstruction error: it measures how well the compressed coordinate \(U^\top X_t\) can recover the componentwise score target. The second term is a posterior-compression error: it measures the loss caused by replacing the full posterior weight \(\pi_i(X_t,t)\) with the projected posterior weight \(\pi_i^U(U^\top X_t,t)\). This posterior-compression term is specific to mixed training and has no analogue in the single-domain frozen-transfer bound.

\noindent\textbf{Regularity of the componentwise comparators.}
We first record how the choice of \(U\) affects the regularity of the componentwise comparators. Recall that
\begin{align*}
\psi_{i,U}(z,t):=H_i(U)g_i(H_i(U)^\dagger z,t),\qquad H_i(U):=U^\top A_i,\quad i=1,2.
\end{align*}
By Assumption~\ref{asm:s_paral_Lip}, the map \(f_i^*(\cdot,t)\) is \(\beta_i\)-Lipschitz uniformly in \(t\). Since
\begin{align*}
g_i(y,t)=\frac{h(t)}{\tilde h_i(t)}f_i^*(y,t)+\left(1-\frac{h(t)}{\tilde h_i(t)}\right)y,
\end{align*}
and \(0\leqslant h(t)/\tilde h_i(t)\leqslant1\), it follows that \(g_i(\cdot,t)\) is \((\beta_i+1)\)-Lipschitz uniformly in \(t\). Hence, for any \(z,z'\in\mathbb R^m\),
\begin{align*}
\|\psi_{i,U}(z,t)-\psi_{i,U}(z',t)\|_2\leqslant L_{\psi,i}(U)\|z-z'\|_2,
\end{align*}
where
\begin{align*}
L_{\psi,i}(U):=\|H_i(U)\|_{\rm op}\|H_i(U)^\dagger\|_{\rm op}(\beta_i+1).
\end{align*}
Moreover, if
\begin{align*}
G_{i,0}:=\sup_{t\in[t_0,T]}\|g_i(0,t)\|_2<\infty,
\end{align*}
then on the box \(\|z\|_\infty\leqslant R\),
\begin{align*}
\|\psi_{i,U}(z,t)\|_2\leqslant K_{\psi,i,U}(R):=\|H_i(U)\|_{\rm op}G_{i,0}+L_{\psi,i}(U)\sqrt m R.
\end{align*}
Finally, define the compact-domain time regularity modulus
\begin{align*}
\tau_{\psi,i,U}(R):=\sup_{t\in[t_0,T],\,\|z\|_\infty\leqslant R}\|\partial_t\psi_{i,U}(z,t)\|_2.
\end{align*}
Whenever \(g_i\) is differentiable in \(t\), we have
\begin{align*}
\partial_t\psi_{i,U}(z,t)=H_i(U)\partial_tg_i(H_i(U)^\dagger z,t).
\end{align*}
Thus, if
\begin{align*}
\tau_{g,i}(R):=\sup_{t\in[t_0,T],\,\|y\|_2\leqslant R}\|\partial_tg_i(y,t)\|_2<\infty,
\end{align*}
then
\begin{align*}
\tau_{\psi,i,U}(R)\leqslant \|H_i(U)\|_{\rm op}\tau_{g,i}\left(\|H_i(U)^\dagger\|_{\rm op}\sqrt m R\right).
\end{align*}

\noindent\textbf{Regularity of the projected posterior weights.}
For \(i=1,2\), the projected density \(q_{i,U,t}\) admits the representation
\begin{align*}
q_{i,U,t}(z)=\int_{\mathbb R^{d_i}}\varphi_{m,\tilde h_i(t)}(z-\alpha(t)H_i(U)y)p_{i,z}(y)\,\rmd y,
\end{align*}
where \(\varphi_{m,\tilde h_i(t)}\) is the \(m\)-dimensional Gaussian density with covariance \(\tilde h_i(t)I_m\). Since \(\tilde h_i(t)=\alpha^2(t)\sigma_i^2+h(t)\) is bounded away from zero on \([t_0,T]\), the Gaussian convolution makes \(q_{i,U,t}\) positive and smooth in \(z\). Under the positivity, smoothness, and tail assumptions on \(p_{i,z}\), differentiation under the integral sign gives finite constants on each compact box:
\begin{align*}
m_{q,i,U}(R)&:=\inf_{t\in[t_0,T],\,\|z\|_\infty\leqslant R}q_{i,U,t}(z)>0,\\
L_{q,i,U}^{(z)}(R)&:=\sup_{t\in[t_0,T],\,\|z\|_\infty\leqslant R}\|\nabla_z q_{i,U,t}(z)\|_2<\infty,\\
L_{q,i,U}^{(t)}(R)&:=\sup_{t\in[t_0,T],\,\|z\|_\infty\leqslant R}|\partial_t q_{i,U,t}(z)|<\infty.
\end{align*}
No equality between \(\sigma_1\) and \(\sigma_2\) is required; the constants above may depend on \(\sigma_1,\sigma_2\) separately.

Since
\begin{align*}
\pi_i^U(z,t)=\frac{\omega_iq_{i,U,t}(z)}{\omega_1q_{1,U,t}(z)+\omega_2q_{2,U,t}(z)},
\end{align*}
we have, for the two-component mixture,
\begin{align*}
\nabla_z\pi_1^U(z,t)=\pi_1^U(z,t)\pi_2^U(z,t)\left(\nabla_z\log q_{1,U,t}(z)-\nabla_z\log q_{2,U,t}(z)\right).
\end{align*}
Therefore, on \(\|z\|_\infty\leqslant R\),
\begin{align*}
\|\nabla_z\pi_1^U(z,t)\|_2\leqslant L_{\pi,U}(R),
\end{align*}
where one may take
\begin{align*}
L_{\pi,U}(R):=\frac14\left(\frac{L_{q,1,U}^{(z)}(R)}{m_{q,1,U}(R)}+\frac{L_{q,2,U}^{(z)}(R)}{m_{q,2,U}(R)}\right).
\end{align*}
Similarly,
\begin{align*}
|\partial_t\pi_1^U(z,t)|\leqslant \tau_{\pi,U}(R),
\end{align*}
where
\begin{align*}
\tau_{\pi,U}(R):=\frac14\left(\frac{L_{q,1,U}^{(t)}(R)}{m_{q,1,U}(R)}+\frac{L_{q,2,U}^{(t)}(R)}{m_{q,2,U}(R)}\right).
\end{align*}
The same bounds hold for \(\pi_2^U\), since \(\pi_2^U=1-\pi_1^U\).

\noindent\textbf{Regularity of the projected mixed comparator.}
We now combine the componentwise comparator regularity and the projected posterior regularity. Recall
\begin{align*}
f_{{\rm mix},U}^{\rm comp}(z,t)=\pi_1^U(z,t)\psi_{1,U}(z,t)+\pi_2^U(z,t)\psi_{2,U}(z,t).
\end{align*}
Equivalently,
\begin{align*}
f_{{\rm mix},U}^{\rm comp}(z,t)=\psi_{2,U}(z,t)+\pi_1^U(z,t)\left(\psi_{1,U}(z,t)-\psi_{2,U}(z,t)\right).
\end{align*}
For \(\|z\|_\infty,\|z'\|_\infty\leqslant R\), using the bounds above gives
\begin{align*}
\|f_{{\rm mix},U}^{\rm comp}(z,t)-f_{{\rm mix},U}^{\rm comp}(z',t)\|_2\leqslant L_{{\rm mix},U}(R)\|z-z'\|_2,
\end{align*}
where one may take
\begin{align*}
L_{{\rm mix},U}(R):=L_{\psi,1}(U)+L_{\psi,2}(U)+L_{\pi,U}(R)\left(K_{\psi,1,U}(R)+K_{\psi,2,U}(R)\right).
\end{align*}
Similarly, for \(\|z\|_\infty\leqslant R\) and \(t,t'\in[t_0,T]\),
\begin{align*}
\|f_{{\rm mix},U}^{\rm comp}(z,t)-f_{{\rm mix},U}^{\rm comp}(z,t')\|_2\leqslant \tau_{{\rm mix},U}(R)|t-t'|,
\end{align*}
where one may take
\begin{align*}
\tau_{{\rm mix},U}(R):=\tau_{\psi,1,U}(R)+\tau_{\psi,2,U}(R)+\tau_{\pi,U}(R)\left(K_{\psi,1,U}(R)+K_{\psi,2,U}(R)\right).
\end{align*}
Thus the regularity constants entering the ReLU approximation and covering-number arguments depend on \(U\) through the component conditioning factors \(H_i(U)^\dagger\), the projected posterior regularity constants \(L_{\pi,U}(R),\tau_{\pi,U}(R)\), and the latent dimension \(m\).

\noindent\textbf{Comparison between \(U=V_1\) and \(U=W_k\).}
The preceding bounds make explicit how the projector choice affects the regularity of the projected mixed comparator. If \(U=V_1\) and \(V_1\) is well aligned with \(A_1\), then \(H_1(V_1)=V_1^\top A_1\) is well conditioned, and \(L_{\psi,1}(V_1)\) is moderate. However, if \(A_2\) is far from the source subspace, then \(H_2(V_1)=V_1^\top A_2\) may be rank deficient or ill conditioned. In that case, the component-2 comparator can suffer from large reconstruction error, and the regularity constant \(L_{\psi,2}(V_1)\) may deteriorate through \(\|H_2(V_1)^\dagger\|_{\rm op}\).

By contrast, \(W_k\) is chosen to reduce the joint subspace residual
\begin{align*}
\Gamma_k(W)=\omega_1c_1\|P_W^\perp A_1\|_F^2+\omega_2c_2\|P_W^\perp A_2\|_F^2.
\end{align*}
When \(k\) is large enough for \(W_k\) to capture the important directions of both \(A_1\) and \(A_2\), the matrices \(H_i(W_k)=W_k^\top A_i\) are better conditioned, reducing the component-comparator constants \(L_{\psi,i}(W_k)\) and the component reconstruction part of \(\mathcal I_{\rm mix}(W_k)\). In the ideal case where \(\operatorname{span}(A_i)\subseteq\operatorname{span}(W_k)\), all nonzero singular values of \(H_i(W_k)\) equal one, and the componentwise spatial regularity reduces to the intrinsic constant \(\beta_i+1\).

However, the posterior-weight regularity constants \(L_{\pi,W_k}(R)\) and \(\tau_{\pi,W_k}(R)\) are controlled by the projected densities \(q_{i,W_k,t}\). These constants need not improve monotonically with \(k\): a higher-dimensional projection may preserve more information for distinguishing the two components, but it may also create sharper posterior transitions or smaller projected-density lower bounds on a fixed box. Therefore, the main guaranteed benefit of \(W_k\) is the reduction of the joint structural residual and the potential improvement of component reconstruction conditioning, while the posterior-compression and dimension-dependent statistical costs must be tracked separately.
\end{proof}

\end{document}